\documentclass{article}
\PassOptionsToPackage{numbers, sort}{natbib}

 \usepackage[preprint]{neurips_2026}

\usepackage[utf8]{inputenc} 
\usepackage[T1]{fontenc}    
\usepackage{hyperref}       
\usepackage{url}            
\usepackage{booktabs}       
\usepackage{amsfonts}       
\usepackage{nicefrac}       
\usepackage{microtype}      
\usepackage{xcolor}         

\usepackage[table]{xcolor}
\usepackage{multirow}
\usepackage[most]{tcolorbox}
\usepackage{listings}
\usepackage{amsmath}
\usepackage{enumitem}
\usepackage{multicol}
\usepackage{graphicx}
\usepackage{subcaption}
\usepackage{float}
\usepackage{stfloats}
\usepackage{wrapfig}
\usepackage{geometry}
\usepackage{pifont}
\usepackage{upquote}
\usepackage{amssymb}
\usepackage{amsthm}
\usepackage{tabularx}
\usepackage{makecell}
\usepackage{algorithm}
\usepackage{algpseudocode}

\title{Persistent Visual Memory: Sustaining\\Perception for Deep Generation in LVLMs}

\author{%
Siyuan Huang$^{1,2}$, Xiaoye Qu$^{1}$\protect\setcounter{footnote}{1}\thanks{Corresponding authors.}  , Yafu Li$^{3}$, Tong Zhu$^{1}$, Zefeng He$^{4}$\\\textbf{Muxin Fu$^{5}$, Daizong Liu$^{6}$, Wei-Long Zheng$^{2}$, Yu Cheng$^{3\dagger}$} \\
$^1$Shanghai AI Laboratory,
$^2$Shanghai Jiao Tong University,\\
$^3$The Chinese University of Hong Kong,
$^4$Nanjing University,\\
$^5$Tongji University,
$^6$Wuhan University
}

\theoremstyle{plain}
\newtheorem{theorem}{Theorem}[section]

\theoremstyle{definition}

\theoremstyle{remark}
\newtheorem{remark}[theorem]{Remark}

\begin{document}

\maketitle

\begin{abstract}
While autoregressive Large Vision-Language Models (LVLMs) demonstrate remarkable proficiency in multimodal tasks, they face a ``Visual Signal Dilution'' phenomenon, where the accumulation of textual history expands the attention partition function, causing visual attention to decay inversely with generated sequence length.
To counteract this, we propose \textbf{P}ersistent \textbf{V}isual \textbf{M}emory (\textbf{PVM}), a lightweight learnable module designed to strengthen sustained, on-demand
access to visual evidence.
Integrated as a parallel branch alongside the Feed-Forward Network (FFN) in LVLMs, PVM establishes a distance-agnostic retrieval pathway that directly provides visual embeddings for enhanced visual perception, thereby structurally mitigating the signal suppression inherent to deep generation.
Extensive experiments on Qwen3-VL models demonstrate that PVM brings notable improvements with negligible parameter overhead, delivering consistent average accuracy gains across both 4B and 8B scales, particularly in complex reasoning tasks that demand persistent visual perception.
Furthermore, in-depth analysis reveals that PVM shows improved robustness in longer generations and accelerates internal prediction convergence.
Our code is available at \url{https://github.com/huaixuheqing/PVM}.
\end{abstract}
\section{Introduction}
\label{sec:intro}

The rapid evolution of Large Vision-Language Models (LVLMs)~\cite{comanici2025gemini, an2025llava, bai2025qwen3vltechnicalreport, hong2025glm, team2025kimi, wang2025internvl3} has established a dominant paradigm in multimodal intelligence.  
By bridging vision encoders with Large Language Models (LLMs) via semantic projection, models like the Qwen3-VL~\cite{bai2025qwen3vltechnicalreport} series have achieved remarkable proficiency across diverse tasks, ranging from fine-grained visual perception to complex reasoning~\cite{li2024survey, su2025thinking}. This modeling approach allows visual tokens to be processed seamlessly alongside textual sequences, effectively extending the cognitive scope of LLMs to encompass visual perception. As these capabilities mature, the research frontier is increasingly shifting towards deep generation scenarios, where models are expected to sustain high-level understanding and interaction across extended dialogues and multi-step deductive chains~\cite{dong2025insight, sun2025mitigating, shen2025long}.

However, this expansion into extended contexts faces a critical Visual Signal Dilution. 
In the standard autoregressive framework, visual tokens serve as fixed precursors~\cite{liu2023visual, bai2025qwen3vltechnicalreport, liu2024improved, xiong2024autoregressive}; unlike text which naturally self-replenishes during generation, they cannot be intrinsically regenerated. 
Our analysis reveals that the attention mechanism inherently creates a structural conflict:
as textual history accumulates, the normalization induced by attention over an ever-growing context redistributes probability mass across more tokens~\cite{vasylenko2025long, liusieve}, causing the once-injected visual signals to be progressively attenuated.
This process drives the model through a phase of asymptotic decay into a Low-Attention Equilibrium, where visual cues are outweighed by textual priors by orders of magnitude~\cite{li2023evaluating, liu2026vision, hu2025enhancing}. Consequently, maintaining fidelity requires a shift from passive retention to sustained, on-demand perception, enabling the model to dynamically re-examine visual memory~\cite{zhou2025learning, long2025seeing, liu2025memverse}.

\begin{figure}[!tbp]
    \centering
    \includegraphics[width=0.9\linewidth]{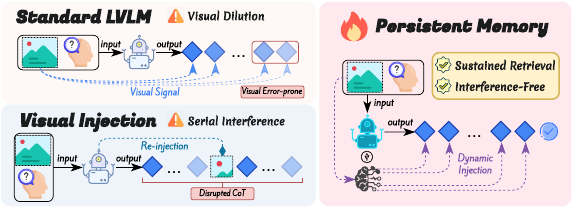}
    \caption{\textbf{Visual Memory Mechanisms.} Unlike Standard LVLMs that degrade via \textit{visual dilution} or Injection methods that cause \textit{serial interference}, our \textbf{P}ersistent \textbf{V}isual \textbf{M}emory (\textbf{PVM}) establishes an independent retrieval path, preserving visual intensity without disrupting the autoregressive flow.}
    \label{fig:intro}
    \vspace{-1mm}
\end{figure}

To counteract this decay, prior works have attempted to explicitly re-inject visual information, a paradigm we characterize as Visual Injection in Figure~\ref{fig:intro}. Whether by inserting raw tokens~\cite{gao2025interleaved, wang2025vgr} or fusing processed features~\cite{yu2025vismem, xing2025boosting}, these approaches typically inject visual retrieval signals directly into the serial autoregressive pathway. While effective for short-term recall, they can disrupt linguistic coherence. By perturbing the evolving semantic states used for step-by-step reasoning, such interventions create a trade-off: reinforcing visual presence comes at the cost of disturbing the precise logical states required for complex deduction~\cite{wang2025vptracker, tian2025large}.

We bridge this gap by proposing \textbf{P}ersistent \textbf{V}isual \textbf{M}emory (\textbf{PVM}), a lightweight learnable module designed to provide sustained visual perception within the autoregressive flow (as shown in Figure~\ref{fig:intro}). 
Structurally, PVM is integrated alongside the Feed-Forward Network (FFN) in the Transformer block, effectively bifurcating the generation flow: while the original FFN preserves the model's reasoning logic, the parallel PVM branch serves as a dedicated channel for retrieving raw visual evidence~\cite{zou2024look, geva2021transformer}.
By employing a gated cross-attention mechanism that attends exclusively to the original visual embeddings, PVM establishes a distance-agnostic retrieval pathway. 
This design structurally mitigates visual signal suppression from the growing textual history, helping the model maintain access to high-fidelity image details throughout extended contexts.

To validate the effectiveness of our proposed PVM, we conduct extensive experiments across a suite of eight challenging multimodal benchmarks, covering general understanding and scientific reasoning. Based on the Qwen3-VL series, our 8B variant with PVM achieves a remarkable 4.8\% average accuracy improvement over the baseline, while demonstrating superior performance among open-source models. 
Notably, our PVM also delivers a consistent 4.4\% improvement when applied to the 4B architecture.
Furthermore, PVM shows improved robustness in extended generation, and mechanistic analysis via LogitLens~\cite{nostalgebraist2020logitlens, cheng2026conditional} suggests that it accelerates internal prediction convergence rather than merely increasing model capacity.

To sum up, our main contributions are threefold:
\begin{itemize}
    \item We formally analyze and empirically validate the phenomenon of {Visual Signal Dilution}, mathematically deriving the asymptotic decay of visual attention and
    uncovering the structural conflict that reveals the necessity for a sustained perception paradigm.
    \item We propose PVM, a lightweight parallel branch that enables sustained, distance-agnostic visual retrieval. We theoretically characterize how this architecture structurally mitigates visual signal suppression and improves access to visual evidence in deep generation.
    \item We achieve substantial gains across eight benchmarks, demonstrating robust scalability on both 4B ({+4.4\%}) and 8B ({+4.8\%}) backbones.
    Further analyses suggest that PVM is more beneficial in longer generations and accelerates internal prediction convergence.
\end{itemize}
\section{Related Work}

\paragraph{General LVLMs and Challenges in Visual Persistence.}
The rapid evolution of Large Language Models (LLMs)~\cite{brown2020language, achiam2023gpt, grattafiori2024llama, yang2025qwen3} has catalyzed the development of Large Vision-Language Models (LVLMs). By bridging pre-trained visual encoders~\cite{radford2021learning, tschannen2025siglip} with LLMs via learnable interfaces, ranging from linear layers~\cite{chen2024internvl, liu2023visual, zhu2023minigpt} and Q-Formers~\cite{li2023blip, bai2023qwen, dai2023instructblip} to more advanced variants~\cite{wang2025internvl3, bai2025qwen3vltechnicalreport}, these models have achieved remarkable proficiency across diverse tasks, including visual perception~\cite{fu2025mmecomprehensiveevaluationbenchmark, liu2024mmbench}, complex reasoning~\cite{qiao2025we, zhang2024mathverse, huang2025spotlight, he2025framethinker, wang2025sampling, he2025diffthinker, he2025videossr}, and open-ended dialogue~\cite{wang2025characterbox, agrawal2024pixtral}. However, despite these advancements, ensuring faithfulness in deep generation remains a formidable challenge; as generation extends into longer responses, LVLMs exhibit a pronounced susceptibility to hallucinations~\cite{bai2024hallucination, liu2024survey}, where generated content diverges from visual facts. 
Rather than being a mere training artifact, this degradation is increasingly recognized as a fundamental architectural bottleneck: as the autoregressive history grows, the static visual tokens located at the sequence start are statistically drowned out by the expanding textual priors~\cite{li2023evaluating, liu2026vision, hu2025enhancing}, creating an urgent need for mechanisms that can sustain visual perception over long horizons.

\paragraph{Visual Injection and Context Management.}
Addressing the volatility of visual signals in extended contexts, recent research has diverged into three primary streams. Hierarchical systems and retrieval-augmented frameworks seek to decouple storage from processing through dual-stream architectures or recurrent bridges~\cite{yu2025vismem, han2025contextuallvlm, lu2025vidove, liu2025comemo, wang2024videollamb, balavzevic2024memory, yuan2025memory, fu2026latentmem, he2026gems}. To manage computational constraints, visual token compression and bottleneck mechanisms reduce feature granularity via resampling~\cite{bulat2025fwd2bot, feng2025vision, shen2025context}, often extended to continual learning scenarios~\cite{torne2025learning}. Additionally, attention optimization strategies employ sparse matrices and decay resilience techniques to mitigate long-range hallucination~\cite{lu2025mdsam, gao2025remember, wu2025chainmpq}. However, these approaches often necessitate complex external modules, structurally alter the backbone, or inevitably sacrifice fine-grained visual fidelity. In contrast, we propose PVM, a parallel retrieval mechanism that preserves original visual features without disrupting the autoregressive flow.
\section{Analysis of Visual Signal Dilution}
\label{sec:problem_analysis}

In this section, we analyze the architectural bottleneck in autoregressive LVLMs. We provide a theoretical framework describing how standard attention inherently dilutes visual signals in Section~\ref{subsec:theory}, and empirically validate these dynamics to offer insights for our method design in Section~\ref{subsec:empirical}.

\subsection{Theoretical Formulation}
\label{subsec:theory}

Consider an autoregressive LVLM generating a new token. The input context is partitioned into a fixed set of visual tokens $\mathcal{V}$ ($|\mathcal{V}|=M$) and a dynamically growing set of textual history tokens $\mathcal{T}_t$ ($|\mathcal{T}_t|=t$).
For the current query vector $\mathbf{q}_t$ in a Transformer block~\cite{vaswani2017attention}, let $s_k(\mathbf{q}_t) = (\mathbf{q}_t^\top \mathbf{k}_k)/\sqrt{d}$ be the unnormalized attention score assigned to the $k$-th context token, where $\mathbf{k}_k$ is the key vector and $d$ is the head dimension. 
To dissect the attention mechanism, we decompose the partition function (the Softmax denominator) into the aggregate unnormalized mass of the visual context $Z_{\mathcal{V}}$ and the textual history $Z_{\mathcal{T}}$: $Z_{\mathcal{V}}(\mathbf{q}_t) = \sum_{k \in \mathcal{V}} \exp(s_k(\mathbf{q}_t)), \, Z_{\mathcal{T}}(\mathbf{q}_t, t) = \sum_{k \in \mathcal{T}_t} \exp(s_k(\mathbf{q}_t))$. We explicitly parameterize $Z_{\mathcal{T}}$ with $t$ to denote the underlying cardinality growth of the summation over time.

In standard self-attention, the probability weight $\alpha_{t,k}$ for the $k$-th token is given by $\exp(s_k(\mathbf{q}_t))/(Z_{\mathcal{V}}(\mathbf{q}_t) + Z_{\mathcal{T}}(\mathbf{q}_t, t))$. We define the \textbf{Visual Attention Mass}, $\Omega_{\mathcal{V}}(t)$, as the aggregate weight allocated to the visual set:
\begin{equation}
\label{eq:standard_attn}
\Omega_{\mathcal{V}}(t) := \sum_{k \in \mathcal{V}} \alpha_{t,k} = \frac{\sum_{k \in \mathcal{V}}\exp(s_k(\mathbf{q}_t))}{Z_{\mathcal{V}}(\mathbf{q}_t) + Z_{\mathcal{T}}(\mathbf{q}_t, t)} = \frac{Z_{\mathcal{V}}(\mathbf{q}_t)}{Z_{\mathcal{V}}(\mathbf{q}_t) + Z_{\mathcal{T}}(\mathbf{q}_t, t)}
\end{equation}
While $Z_{\mathcal{V}}(\mathbf{q}_t)$ is bounded by the fixed number of visual tokens $M$, $Z_{\mathcal{T}}(\mathbf{q}_t, t)$ aggregates over a growing textual history $t$. We characterize the asymptotic behavior of $\Omega_{\mathcal{V}}(t)$ through two distinct operational phases: the \textit{Power-Law Dilution Phase} and the \textit{High-Magnitude Saturation Phase}.

\subsubsection{Phase I: Power-Law Dilution via Active Competition}
\label{subsubsec:phase1}

In the early-to-mid stage of generation, the model must continuously reference historical context to maintain coherence. We posit a \textit{Persistent textual attention} condition: historical text tokens maintain a statistically significant relevance, bounded below by a positive constant $\mu > 0$ such that the average unnormalized attention mass satisfies $t^{-1} Z_{\mathcal{T}}(\mathbf{q}_t, t) \ge \mu$.

However, unlike the growing textual history, the visual information is constrained by the fixed number of visual tokens $M$. In practice, input normalization constrains the pre-softmax scores, effectively imposing an upper bound $s_{\max}$. Consequently, the aggregate visual mass $Z_{\mathcal{V}}(\mathbf{q}_t)$ is upper-bounded by $\beta = M \cdot \exp(s_{\max})$. Based on this structural conflict, we formulate the dilution theorem:
\begin{theorem}[\textbf{Visual Signal Dilution}]
\label{thm:dilution}
Given fixed visual context size $M$ and textual attention lower bound $\mu$, the visual attention mass $\Omega_{\mathcal{V}}(t)$ is structurally dominated by the growing textual history:
\begin{equation}
    \Omega_{\mathcal{V}}(t) \le \frac{\beta}{\beta + \mu \cdot t} = \mathcal{O}(t^{-1})
\end{equation}
\end{theorem}

\begin{proof}
Substituting the visual upper bound $Z_{\mathcal{V}} \le \beta$ and the textual lower bound $Z_{\mathcal{T}} \ge \mu t$ into Eq.~\ref{eq:standard_attn} yields the governing inequality: $\Omega_{\mathcal{V}}(t) = \frac{Z_{\mathcal{V}}}{Z_{\mathcal{V}} + Z_{\mathcal{T}}} \le \frac{\beta}{\beta + Z_{\mathcal{T}}} \le \frac{\beta}{\beta + \mu \cdot t}$.
As $t$ grows, the linear term $\mu \cdot t$ in the denominator dominates the constant $\beta$, implying $\Omega_{\mathcal{V}}(t)=\mathcal{O}(t^{-1})$.
\end{proof}

\subsubsection{Phase II: The High-Magnitude Saturation Trap}
\label{subsubsec:phase2}

In practice, $Z_{\mathcal{T}}$ eventually saturates due to effective attention limits. However, we argue that this saturation does not rescue the visual signal; rather, it locks the model into a Low-Attention Equilibrium.
Let $W_{eff}$ denote the effective textual window size. Since the effective textual window is typically orders of magnitude larger than the visual context ($W_{eff} \gg M$), the aggregate textual mass stabilizes at a high plateau $Z_{\mathcal{T}}^{sat}$. Thus, the system converges to: $\lim_{t \to \infty} \Omega_{\mathcal{V}}(t) \approx \frac{\mathbb{E}[Z_{\mathcal{V}}]}{Z_{\mathcal{T}}^{sat}} \ll 1$.

Although the dilution halts, the system stabilizes with a heavily skewed distribution. The visual signal remains structurally outweighed by the overwhelming textual momentum, making the model more prone to errors when precise visual grounding is required.

\subsection{Empirical Verification}
\label{subsec:empirical}

To validate our theory, we analyze attention dynamics using Qwen3-VL-8B-Instruct~\cite{bai2025qwen3vltechnicalreport} on a subset of the COCO 2017~\cite{lin2014microsoft} validation set. We designed a ``Blind Painter'' Stress Test (see Appendix~\ref{sec:appendix_prompt}). Specifically, we prompt the model to reconstruct the image textually with extreme detail, as if instructing a painter who cannot see the original work. This setup necessitates active visual retrieval at every generation step rather than just the beginning, creating a rigorous testbed that isolates signal dilution from general language generation.

\begin{figure}[!htbp]
    \centering
    \begin{minipage}[t]{0.48\textwidth}
        \centering
        \includegraphics[height=3.4cm, keepaspectratio]{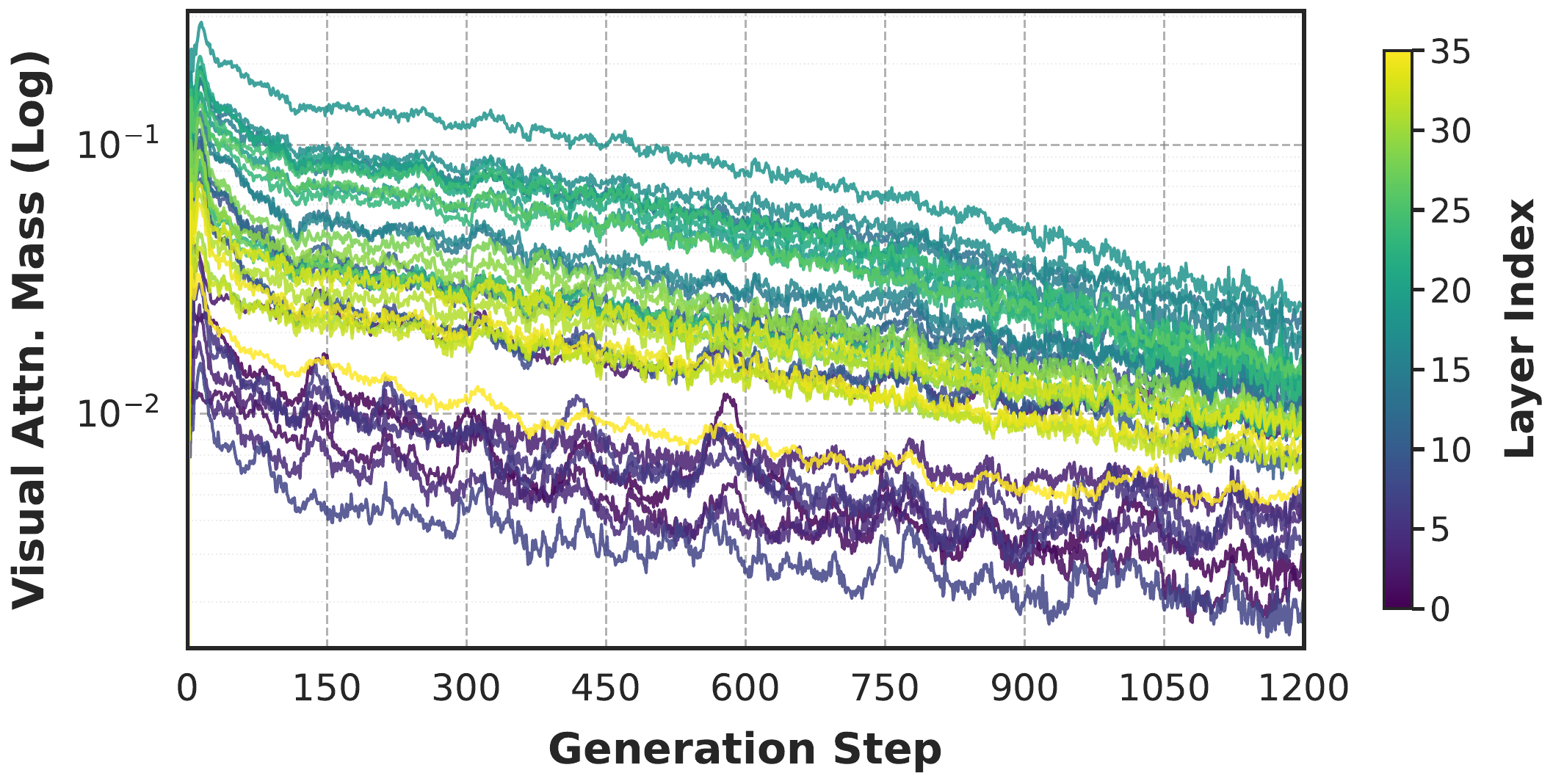}
        \caption{\textbf{Power-Law Decay of Visual Signal.} 
        Log-scale analysis shows that $\Omega_{\mathcal{V}}$ approximately follows the $\mathcal{O}(t^{-1})$ trend predicted by Theorem~\ref{thm:dilution}. 
        This demonstrates that visual attention mass is structurally diluted, decaying inversely to the sequence length $t$.}
        \label{fig:decay_trend}
    \end{minipage}\hfill 
    \begin{minipage}[t]{0.48\textwidth}
        \centering
        \includegraphics[height=3.4cm, keepaspectratio]{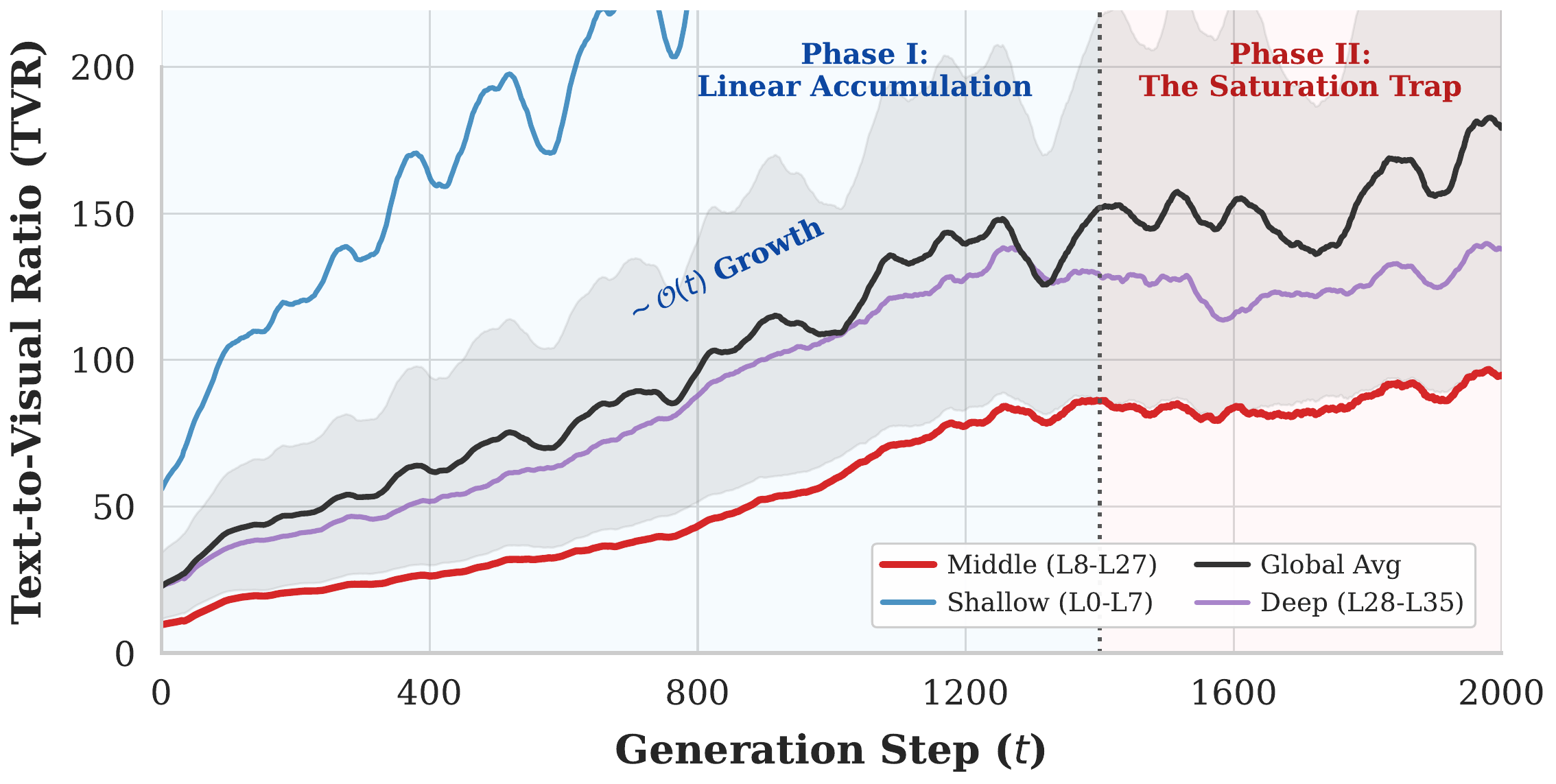}
        \caption{\textbf{Evolution of Textual Dominance.} 
        The TVR trajectory is consistent with the proposed two-phase dilution mechanism.
        This reveals how initial probability competition saturates into a strict equilibrium, where textual priors structurally overwhelm visual signals.}
        \label{fig:tvr_curve}
    \end{minipage}
\end{figure}

\subsubsection{The Phenomenon: Spatiotemporal Decay of Visual Mass}

We track the Visual Attention Mass $\Omega_{\mathcal{V}}$ throughout the generation process. As shown in the logarithmic analysis (Figure~\ref{fig:decay_trend}), the empirical decay trajectory closely aligns with the $\mathcal{O}(t^{-1})$ power-law predicted by Theorem~\ref{thm:dilution}. 
Crucially, layer-wise analysis reveals that this attenuation is pronounced in the intermediate layers (Layers 8--27), which serve as the primary locus for multimodal reasoning (detailed in Appendix~\ref{app:visual_attention}). Thus, even though these layers retain the highest relative visual relevance, they are forced to operate on a progressively weakening visual foundation as the text expands, motivating a mechanism for active, on-demand visual retrieval.

\definecolor{myblue}{rgb}{0.21,0.49,0.74}
\begin{tcolorbox}[colframe=black!50, colback=myblue!8, boxrule=1.5pt, arc=2mm, top=4pt, bottom=4pt, left=4pt, right=4pt,  boxsep=1pt]
\textit{\textbf{Insight 1}: Since the autoregressive stream tends to dilute visual signals, we should establish a parallel memory pathway that is structurally isolated from the main reasoning backbone.}
\end{tcolorbox}

\subsubsection{The Mechanism: Dynamics of Textual Dominance}

To elucidate the mathematical driver of attenuation, we analyze the Text-to-Visual Ratio (TVR): $\text{TVR}(t) = Z_{\mathcal{T}}(\mathbf{q}_t, t) / Z_{\mathcal{V}}(\mathbf{q}_t)$. Figure~\ref{fig:tvr_curve} validates the two-phase dynamics: the trajectory initially exhibits linear growth (\textit{Phase I}), suggesting that historical text acts as an active competitor for probability mass. Crucially, it eventually stabilizes at a saturation plateau (\textit{Phase II}) where the textual mass exceeds the visual mass by nearly two orders of magnitude. 
This quantifies the Low-Attention Equilibrium: at such high ratios, the visual signal $Z_{\mathcal{V}}$ is statistically insignificant in the shared partition function, limiting the model's ability to maintain strict visual fidelity in deep generation.

\begin{tcolorbox}[colframe=black!50, colback=myblue!8, boxrule=1.5pt, arc=2mm, top=4pt, bottom=4pt, left=4pt, right=4pt,  boxsep=1pt]
\textit{\textbf{Insight 2}: To prevent textual dominance in the partition function, the retrieval mechanism requires independent attention normalization, confined entirely to the closed visual domain.}
\end{tcolorbox}
\section{Method: Persistent Visual Memory}
\label{sec:method}

Guided by the theoretical analysis and insights in Section~\ref{sec:problem_analysis}, we propose \textbf{P}ersistent \textbf{V}isual \textbf{M}emory (\textbf{PVM}). As shown in Figure~\ref{fig:PVM}, this framework operationalizes the insights derived from the dilution bottleneck: it satisfies the need for a \textit{parallel memory pathway} and enforces \textit{independent attention normalization}.
In Section~\ref{sec:pvm_arch}, we detail the parallel architecture and its gated fusion mechanism. Subsequently, in Section~\ref{sec:guarantee}, we provide a theoretical analysis showing that PVM structurally mitigates the impact of growing textual context on visual retrieval.

\subsection{Architecture Design}
\label{sec:pvm_arch}

Unlike previous approaches relying on external retrieval~\cite{yasunaga2022retrieval, caffagni2024wiki, mei2025survey}, we integrate PVM into the Transformer decoder block as a parallel branch. 
Specifically, inspired by MemVR~\cite{zou2024look} and the understanding of Feed-Forward Networks as key-value memories~\cite{geva2021transformer}, 
we position the PVM parallel to the FFN sub-layer to function as an active perception channel for the current visual context.

Let $\mathbf{x} \in \mathbb{R}^{d}$ denote the hidden state output from the Multi-Head Self-Attention (MHSA) layer. 
Departing from the standard serial architecture where $\mathbf{x}$ is fed exclusively into the FFN, our design bifurcates the information flow into two parallel streams:
\begin{itemize}[leftmargin=*]
    \item \textbf{The Reasoning Path (Original):} The hidden state travels through the frozen FFN to access pre-trained static knowledge and logical patterns: $\mathbf{h}_{\mathrm{ffn}} = \mathrm{FFN}(\mathbf{x})$.
    \item \textbf{The Looking Path (PVM):} Simultaneously, the same state $\mathbf{x}$ acts as the \textit{Query} for the PVM module. This establishes the \textit{parallel memory pathway} that is structurally isolated from the autoregressive bottleneck, allowing the model to actively retrieve visual details.
\end{itemize}

\paragraph{PVM Computation.}
To ensure parameter efficiency and minimize inference overhead, PVM serves as a bottleneck adapter in a projected latent space rather than the full model dimension $d$.
Let $\mathbf{x} \in \mathbb{R}^{d}$ be the input hidden state and $\mathbf{V}_{\mathrm{img}} \in \mathbb{R}^{M \times d}$ represent the features of the fixed visual set $\mathcal{V}$. The computation proceeds in three stages:
\begin{enumerate}[leftmargin=*]
    \item \textbf{Projection:} We project the input hidden state and visual features into a lower-dimensional latent space $d' < d$ via two independent learnable reducers, $\mathbf{W}_{\mathrm{down}}^{\mathrm{txt}}$ and $\mathbf{W}_{\mathrm{down}}^{\mathrm{vis}} \in \mathbb{R}^{d \times d'}$: $\mathbf{x}_{\mathrm{lat}} = \mathbf{x} \mathbf{W}_{\mathrm{down}}^{\mathrm{txt}}, \quad \mathbf{V}_{\mathrm{lat}} = \mathbf{V}_{\mathrm{img}} \mathbf{W}_{\mathrm{down}}^{\mathrm{vis}}$.
    
    \item \textbf{Latent Retrieval:} We perform Cross-Attention where the projected hidden state serves as \textit{Query} and the projected visual features $\mathbf{V}_{\mathrm{lat}}$ serve as both \textit{Keys} and \textit{Values}, followed by a lightweight FFN.
    Crucially, this restricts the attention domain solely to the visual set $\mathcal{V}$, realizing the \textit{independent attention normalization}: $\mathbf{h}_{\mathrm{attn}} = \mathrm{CrossAttn}(Q=\mathbf{x}_{\mathrm{lat}}, K=\mathbf{V}_{\mathrm{lat}}, V=\mathbf{V}_{\mathrm{lat}})$, $\mathbf{h}_{\mathrm{lat}} = \mathbf{h}_{\mathrm{attn}} + \mathrm{FFN}_{\mathrm{lat}}(\mathrm{RMSNorm}(\mathbf{h}_{\mathrm{attn}}))$.
    
    \item \textbf{Restoration:} The refined latent feature is projected back to the original high-dimensional space via an up-projection matrix $\mathbf{W}_{\mathrm{up}} \in \mathbb{R}^{d' \times d}$: $\mathbf{h}_{\mathrm{pvm}} = \mathbf{h}_{\mathrm{lat}} \mathbf{W}_{\mathrm{up}}$.
\end{enumerate}
\paragraph{Gated Fusion with Selective Activation.}
We integrate the visual memory into the main stream via a residual connection controlled by a learnable scalar gate $\lambda$. 
To preserve the integrity of visual representations, we employ a Visual Silencing Mask $\mathcal{M}_{\mathrm{txt}}$, a binary indicator that activates only for text tokens.
The final output $\mathbf{y}$ is computed as:
\begin{equation}
    \mathbf{y} = \mathbf{x} + \mathbf{h}_{\mathrm{ffn}} + \underbrace{(\lambda \cdot \mathbf{h}_{\mathrm{pvm}}) \cdot \mathcal{M}_{\mathrm{txt}}}_{\text{Active Visual Injection}}
\end{equation}
where $\lambda$ is initialized to 0 to preserve pre-trained capabilities. 
Crucially, the fused representation $\mathbf{y}$ maintains the shape of $\mathbf{x}$, allowing seamless integration into the backbone without structural changes.

\begin{figure*}[t]
    \centering
    \includegraphics[width=1\textwidth]{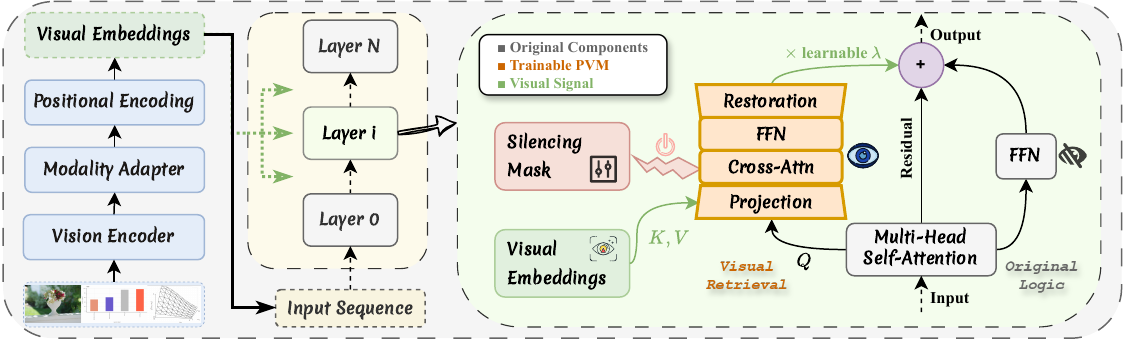} 
    \caption{\textbf{Overview of the Persistent Visual Memory (PVM) framework.} PVM is integrated parallel to the frozen FFN to shield active visual retrieval from sequential dilution during autoregressive reasoning. It treats the hidden state as a query to retrieve specific visual contexts via a parameter-efficient bottleneck adapter (consisting of Projection, Cross-Attn, FFN, and Restoration). The module employs a Silencing Mask to activate selectively during text generation, injecting the retrieved visual signal into the main stream through a learnable gate to enhance visual grounding.}
    \label{fig:PVM}
\end{figure*}

\subsection{Structural Analysis of PVM}
\label{sec:guarantee}

We conclude our method description by formally characterizing how PVM addresses the bottleneck identified in Section~\ref{sec:problem_analysis}. 
Recall that Theorem~\ref{thm:dilution} attributes visual signal dilution to the contamination of the partition function by the growing textual mass $Z_{\mathcal{T}}$. 
PVM implements \textit{independent attention normalization} to structurally mitigate this dependency.

\paragraph{Decoupled Partition Function.}
In the PVM branch, the attention operation is confined to the visual domain. Given an input hidden state $\mathbf{x}$, the probability mass assigned to the $k$-th visual token is:
\begin{equation}
    \label{eq:pvm_normalization}
    \beta_{k}(\mathbf{x}) = \frac{\exp(\mathbf{x} \mathbf{W}_Q (\mathbf{v}_k \mathbf{W}_K)^\top / \sqrt{d'})}{Z_{\mathrm{pvm}}(\mathbf{x})}
\end{equation}
where $Z_{\mathrm{pvm}}(\mathbf{x}) = \sum_{j \in \mathcal{V}} \exp(\mathbf{x} \mathbf{W}_Q (\mathbf{v}_j \mathbf{W}_K)^\top / \sqrt{d'})$. Crucially, the summation in $Z_{\mathrm{pvm}}$ is bounded exclusively by the fixed visual set $\mathcal{V}$.
To mathematically isolate this structural decoupling from the evolving nature of autoregressive queries, we adopt a fixed local query assumption for the hidden state $\mathbf{x}$ (detailed in Appendix~\ref{sec:appendix_approximation}).

\begin{theorem}[\textbf{Local Structural Mitigation of Visual Signal Dilution}]
\label{thm:immunity}
Under the PVM architecture, and conditioning on a fixed local hidden state $\mathbf{x}$, the retrieval representation $\mathbf{h}_{\mathrm{pvm}}$ is decoupled from the textual-history length $t$ in its partition function. Specifically, it satisfies the following local invariance property: $\frac{\partial \|\mathbf{h}_{\mathrm{pvm}}\|}{\partial t} = 0$.
This stands in contrast to the standard backbone, where visual relevance decays asymptotically as $\Omega_{\mathcal{V}}(t) \in \mathcal{O}(t^{-1})$ (Theorem~\ref{thm:dilution}).
\end{theorem}

\begin{proof}
Let $\psi_k = \frac{1}{\sqrt{d'}} \mathbf{x} \mathbf{W}_Q (\mathbf{v}_k \mathbf{W}_K)^\top$ denote the unnormalized attention score. The PVM output is explicitly formulated as:
\begin{equation}
    \mathbf{h}_{\mathrm{pvm}} = \left( \sum_{k \in \mathcal{V}} \frac{\exp(\psi_k)}{Z_{\mathrm{pvm}}(\mathbf{x})} \mathbf{v}_k \mathbf{W}_{V} \right) \mathbf{W}_{\mathrm{up}}
\end{equation}
The partition function $Z_{\mathrm{pvm}}(\mathbf{x}) = \sum_{j \in \mathcal{V}} \exp(\psi_j)$ sums over the fixed visual set $\mathcal{V}$ with constant cardinality $M$. 
 Given the static nature of the visual set $\mathcal{V}$ and conditioning on a fixed local hidden state $\mathbf{x}$, the variable $t$ appears neither in the score computation $\psi_k$ nor in the summation index $\mathcal{V}$, and is therefore algebraically absent from the definition of $\mathbf{h}_{\mathrm{pvm}}$.
Thus, the partial derivative with respect to the sequence length vanishes locally: $\frac{\partial \|\mathbf{h}_{\mathrm{pvm}}\|}{\partial t} = 0$.
\end{proof}

\begin{remark}
Theorem~\ref{thm:immunity} validates the design principles proposed in Section~\ref{subsec:empirical} under the assumption stated above:
\begin{itemize}[leftmargin=*]
    \item \textbf{Realization of Insight 1 (Parallel Memory Pathway).} 
    The algebraic absence of $t$ from the PVM partition function confirms the structural isolation of the visual stream. Unlike the entangled backbone, this parallel design establishes a dedicated channel for visual grounding, thereby supporting stronger perceptual capabilities.
    
    \item \textbf{Realization of Insight 2 (Independent Attention Normalization).} 
    By restricting the partition function to the fixed visual set, the PVM branch avoids direct probability competition from growing text. This structurally mitigates the dilution effect that arises in standard attention.
\end{itemize}
\end{remark}
\section{Experiments}
\label{sec:experiments}

\paragraph{Models and Data.}
We adopt the Qwen3-VL-Instruct~\cite{bai2025qwen3vltechnicalreport} series (4B and 8B variants) as our backbones. To ensure minimal overhead, PVM modules are integrated into selected intermediate Transformer layers aligning with the configuration in DeepStack~\cite{meng2024deepstack}, introducing only 27.92M additional trainable parameters (a negligible $\sim$0.32\% of the 8B model size). 
Our training data comprises two subsets: a supervised fine-tuning set $\mathcal{D}_{\text{sft}}$ of 526k samples from OpenMMReasoner-SFT~\cite{zhang2025openmmreasoner}, and a reinforcement learning set $\mathcal{D}_{\text{rl}}$ of 3.6k complex reasoning queries aggregated from MMK12~\cite{meng2025mm}, ThinkLite-VL-hard~\cite{wang2025sota}, ViRL39K~\cite{wang2025vl}, and We-Math2.0-Pro~\cite{qiao2025we}.

\paragraph{Training Details.}
Our pipeline contains two stages:
\textbf{Stage I: Visual Memory Alignment (SFT).} We freeze the backbone and exclusively optimize the PVM modules and gating scalars to establish the semantic mapping between textual queries and visual keys.
\textbf{Stage II: Policy Refinement (GRPO).} Using Group Relative Policy Optimization~\cite{shao2024deepseekmath}, we unfreeze the LLM backbone and PVM modules (keeping the Vision Encoder frozen) to enforce active visual retrieval in complex reasoning.
Further details on model configuration, data curation, and hyperparameters are provided in Appendix~\ref{app:training_details}.

\paragraph{Baselines.}
We benchmark PVM against three distinct categories of methods: 
(1) \textbf{Internal Baselines:} We compare against the original Qwen3-VL alongside its SFT, LoRA~\cite{hu2022lora}, and RL-enhanced variants trained on our datasets. This rigorously isolates PVM's architectural benefits from gains driven purely by high-quality data or RL alignment.
(2) \textbf{Visual Injection Methods:} We include three representative paradigms: MemVR~\cite{zou2024look} for uncertainty-driven visual retracing, ICoT~\cite{gao2025interleaved} for attention-based token interleaving, and CoMemo~\cite{liu2025comemo} for spatially-aware visual retention. To ensure a rigorous comparison, we re-implemented all these methods on the identical Qwen3-VL-8B-Instruct backbone, thereby isolating algorithmic efficacy from architectural discrepancies.
(3) \textbf{RL Reasoners:} To benchmark our performance in complex reasoning, we compare against leading multimodal reasoning models, specifically Euclid-8B~\cite{lian2025euclid}, PEARL-8B~\cite{zhang2025perceptual}, and OneThinker-8B~\cite{feng2025onethinker}.

\paragraph{Evaluation Benchmarks.}
We conduct a comprehensive evaluation on eight diverse multimodal benchmarks, ranging from general multimodal understanding to domain-specific reasoning in mathematics and science, including MMMU~\cite{yue2024mmmu}, MMBench-CN, MMBench-EN~\cite{liu2024mmbench}, MMStar~\cite{chen2024we}, MMT~\cite{ying2024mmt}, MathVerse~\cite{zhang2024mathverse}, MathVision~\cite{wang2024measuring}, and AI2D~\cite{kembhavi2016diagram}. We report 4-run average accuracy at an inference temperature of 0.7, using lmms-eval~\cite{zhang2025lmms} for all evaluations to ensure fair comparison.
\section{Results}
\label{sec:results}

\newcolumntype{Y}{>{\centering\arraybackslash}X}
\definecolor{LightCyan}{rgb}{0.88,1,1}
\definecolor{lightyellow}{RGB}{254,249,229}
\definecolor{ouryellow}{RGB}{252,236,167}
\definecolor{lightgreen}{RGB}{235,241,230}
\definecolor{lightpink}{RGB}{244,234,232}
\colorlet{upgreen}{green!75!black}
\colorlet{downred}{red!75!black}

\newcommand{\basep}[1]{\textsubscript{\scalebox{0.8}{\textcolor{upgreen}{\tiny$\uparrow$#1}}}}
\newcommand{\basem}[1]{\textsubscript{\scalebox{0.8}{\textcolor{downred}{\tiny$\downarrow$#1}}}}

\begin{table}[!t]
\caption{\textbf{Main Results on General and Reasoning Benchmarks.} We report 4-run average accuracy (\%) across 8 benchmarks, with the best and second-best performances highlighted in \textbf{bold} and \underline{underlined}. The numbers in subscripts denote the absolute improvement (\textcolor{upgreen}{$\uparrow$}) or decline (\textcolor{downred}{$\downarrow$}) relative to\colorbox{gray!10}{the original model baseline}. \textbf{PVM (SFT + GRPO)} achieves the best overall performance, demonstrating consistent gains in both comprehensive understanding and specialized reasoning.}
\label{tab:main_results_image}

\centering
\setlength{\tabcolsep}{4pt}
\renewcommand{\arraystretch}{1.00}
\scriptsize

\resizebox{\textwidth}{!}{%
\begin{tabular}{r | c c c c c c | c c c c | c}
\toprule
\multicolumn{1}{c|}{\multirow{2}{*}[-2mm]{\textbf{Model}}} &
\multicolumn{6}{c|}{\textbf{General \& Comprehensive}} &
\multicolumn{4}{c|}{\textbf{Math \& Science}} &
\multicolumn{1}{c}{\multirow{2}{*}[-2mm]{\textbf{Avg.}}} \\
\cmidrule{2-7} \cmidrule{8-11}
& \textbf{MMMU$^\text{dev}$}
& \makecell{\textbf{MMBench}\\\textbf{-CN$^\text{lite}$}}
& \makecell{\textbf{MMBench}\\\textbf{-EN$^\text{lite}$}}
& \textbf{MMStar}
& \textbf{MMT$^\text{emo}$}
& \textbf{Avg.}
& \makecell{\textbf{Math}\\\textbf{Verse$^\text{V}$}}
& \makecell{\textbf{Math}\\\textbf{Vision$^\text{mini}$}}
& \textbf{AI2D$^\text{lite}$}
& \textbf{Avg.}
& \\
\midrule

\rowcolor{lightpink!100} \multicolumn{12}{c}{\textit{Visual Injection Methods}} \\
\midrule
MemVR~\cite{zou2024look} & 59.3\basep{2.0} & 86.4\basep{0.0} & 86.4\basep{0.0} & 65.4\basem{3.3} & 54.2\basem{2.5} & 70.3\basem{0.8} & 52.9\basep{0.0} & 48.4\basep{3.0} & 78.8\basem{1.0} & 60.0\basep{0.6} & 66.5\basem{0.2} \\
ICoT~\cite{gao2025interleaved} & 63.3\basep{6.0} & 88.2\basep{1.8} & \underline{88.6}\basep{2.2} & 69.3\basep{0.6} & 54.2\basem{2.5} & 72.7\basep{1.6} & 56.1\basep{3.2} & 48.1\basep{2.7} & 78.6\basem{1.2} & 60.9\basep{1.5} & 68.3\basep{1.6} \\
CoMemo~\cite{liu2025comemo} & 62.0\basep{4.7} & 88.0\basep{1.6} & \underline{88.6}\basep{2.2} & 69.9\basep{1.2} & \textbf{60.8}\basep{4.1} & 73.9\basep{2.8} & 55.0\basep{2.1} & 43.4\basem{2.0} & 79.6\basem{0.2} & 59.3\basem{0.1} & 68.4\basep{1.7} \\
\midrule

\rowcolor{lightgreen!100} \multicolumn{12}{c}{\textit{Recent RL-tuned models}} \\
\midrule
Euclid-8B~\cite{lian2025euclid} & 62.0\basep{4.7} & \underline{90.4}\basep{4.0} & \underline{88.6}\basep{2.2} & 71.3\basep{2.6} & 54.2\basem{2.5} & 73.3\basep{2.2} & 57.9\basep{5.0} & 49.0\basep{3.6} & \underline{82.6}\basep{2.8} & \underline{63.2}\basep{3.8} & 69.5\basep{2.8} \\
PEARL-8B~\cite{zhang2025perceptual} & 65.3\basep{8.0} & 87.2\basep{0.8} & 87.1\basep{0.7} & \underline{71.5}\basep{2.8} & 55.0\basem{1.7} & 73.2\basep{2.1} & \underline{58.9}\basep{6.0} & 47.7\basep{2.3} & 81.8\basep{2.0} & 62.8\basep{3.4} & 69.3\basep{2.6} \\
OneThinker-8B~\cite{feng2025onethinker} & 62.0\basep{4.7} & 87.2\basep{0.8} & 87.9\basep{1.5} & 71.2\basep{2.5} & 51.7\basem{5.0} & 72.0\basep{0.9} & 57.4\basep{4.5} & 45.4\basep{0.0} & 81.4\basep{1.6} & 61.4\basep{2.0} & 68.0\basep{1.3} \\
\midrule

\rowcolor{ouryellow!60} \multicolumn{12}{c}{\textit{Our comprehensive comparison}} \\
\midrule
\rowcolor{gray!10}
Qwen3-VL-8B-Instruct & 57.3\basep{0.0} & 86.4\basep{0.0} & 86.4\basep{0.0} & 68.7\basep{0.0} & 56.7\basep{0.0} & 71.1\basep{0.0} & 52.9\basep{0.0} & 45.4\basep{0.0} & 79.8\basep{0.0} & 59.4\basep{0.0} & 66.7\basep{0.0} \\
SFT & 60.7\basep{3.4} & 88.0\basep{1.6} & 87.9\basep{1.5} & 67.7\basem{1.0} & 50.8\basem{5.9} & 71.0\basem{0.1} & 56.9\basep{4.0} & 48.0\basep{2.6} & 79.0\basem{0.8} & 61.3\basep{1.9} & 67.4\basep{0.7} \\
LoRA-SFT & 63.3\basep{6.0} & 88.8\basep{2.4} & \underline{88.6}\basep{2.2} & 70.2\basep{1.5} & 51.7\basem{5.0} & 72.5\basep{1.4} & 55.0\basep{2.1} & 42.8\basem{2.6} & 79.8\basep{0.0} & 59.2\basem{0.2} & 67.5\basep{0.8} \\
\rowcolor{LightCyan}
PVM-8B (SFT) & \underline{66.7}\basep{9.4} & \underline{90.4}\basep{4.0} & \textbf{89.4}\basep{3.0} & 71.2\basep{2.5} & \underline{58.3}\basep{1.6} & \underline{75.2}\basep{4.1} & 57.5\basep{4.6} & \underline{50.7}\basep{5.3} & 80.8\basep{1.0} & 63.0\basep{3.6} & \underline{70.6}\basep{3.9} \\
\hline

SFT + GRPO & 60.7\basep{3.4} & 88.8\basep{2.4} & 87.9\basep{1.5} & 68.6\basem{0.1} & 54.2\basem{2.5} & 72.0\basep{0.9} & 58.5\basep{5.6} & 48.0\basep{2.6} & 79.6\basem{0.2} & 62.0\basep{2.6} & 68.3\basep{1.6} \\
LoRA-SFT + GRPO & 64.7\basep{7.4} & 86.4\basep{0.0} & 87.1\basep{0.7} & 71.0\basep{2.3} & 52.5\basem{4.2} & 72.3\basep{1.2} & 57.6\basep{4.7} & 46.7\basep{1.3} & 81.0\basep{1.2} & 61.8\basep{2.4} & 68.4\basep{1.7} \\
\rowcolor{LightCyan}
\textbf{PVM-8B (SFT + GRPO)} & \textbf{67.3}\basep{10.0} & \textbf{91.2}\basep{4.8} & \textbf{89.4}\basep{3.0} & \textbf{71.6}\basep{2.9} & \underline{58.3}\basep{1.6} & \textbf{75.6}\basep{4.5} & \textbf{59.8}\basep{6.9} & \textbf{51.3}\basep{5.9} & \textbf{82.8}\basep{3.0} & \textbf{64.6}\basep{5.2} & \textbf{71.5}\basep{4.8} \\
\midrule
\midrule

\rowcolor{lightyellow!100} \multicolumn{12}{c}{\textit{Results of 4B models}} \\
\midrule
\rowcolor{gray!10}
Qwen3-VL-4B-Instruct & 57.3\basep{0.0} & 86.0\basep{0.0} & 78.8\basep{0.0} & 66.7\basep{0.0} & \underline{56.7}\basep{0.0} & 69.1\basep{0.0} & 52.4\basep{0.0} & 35.9\basep{0.0} & 78.4\basep{0.0} & 55.6\basep{0.0} & 64.0\basep{0.0} \\
SFT & 58.0\basep{0.7} & 87.2\basep{1.2} & 85.6\basep{6.8} & 67.7\basep{1.0} & 50.8\basem{5.9} & 69.9\basep{0.8} & 52.7\basep{0.3} & 37.5\basep{1.6} & 77.6\basem{0.8} & 55.9\basep{0.3} & 64.6\basep{0.6} \\
LoRA-SFT & 56.0\basem{1.3} & 87.2\basep{1.2} & \underline{87.1}\basep{8.3} & 66.7\basep{0.0} & 55.0\basem{1.7} & 70.4\basep{1.3} & 52.4\basep{0.0} & 37.2\basep{1.3} & 79.2\basep{0.8} & 56.3\basep{0.7} & 65.1\basep{1.1} \\
\rowcolor{LightCyan}
PVM-4B (SFT) & \underline{60.7}\basep{3.4} & 88.0\basep{2.0} & \textbf{87.9}\basep{9.1} & 67.9\basep{1.2} & \textbf{57.5}\basep{0.8} & \underline{72.4}\basep{3.3} & 54.4\basep{2.0} & 41.5\basep{5.6} & \underline{80.0}\basep{1.6} & \underline{58.6}\basep{3.0} & \underline{67.2}\basep{3.2} \\
\hline

SFT + GRPO & 58.0\basep{0.7} & 88.0\basep{2.0} & 85.6\basep{6.8} & 65.1\basem{1.6} & 53.3\basem{3.4} & 70.0\basep{0.9} & \underline{54.6}\basep{2.2} & \underline{42.8}\basep{6.9} & 78.6\basep{0.2} & \underline{58.6}\basep{3.0} & 65.7\basep{1.7} \\
LoRA-SFT + GRPO & 56.0\basem{1.3} & \underline{88.8}\basep{2.8} & 84.9\basep{6.1} & \underline{69.0}\basep{2.3} & 55.8\basem{0.9} & 70.9\basep{1.8} & 54.2\basep{1.8} & 42.4\basep{6.5} & 76.8\basem{1.6} & 57.8\basep{2.2} & 66.0\basep{2.0} \\
\rowcolor{LightCyan}
\textbf{PVM-4B (SFT + GRPO)} & \textbf{62.7}\basep{5.4} & \textbf{90.4}\basep{4.4} & \textbf{87.9}\basep{9.1} & \textbf{69.2}\basep{2.5} & 55.8\basem{0.9} & \textbf{73.2}\basep{4.1} & \textbf{55.0}\basep{2.6} & \textbf{45.4}\basep{9.5} & \textbf{81.0}\basep{2.6} & \textbf{60.4}\basep{4.8} & \textbf{68.4}\basep{4.4} \\
\bottomrule
\end{tabular}%
}
\end{table}

\subsection{Main Results}
\label{subsec:main_results}

Table~\ref{tab:main_results_image} benchmarks our approach against internal baselines, specialized visual injection methods, and recent RL-tuned models. 
On the 8B backbone, PVM-SFT demonstrates strong performance with an overall score of 70.6\%, outperforming vanilla SFT, LoRA-SFT, and existing methods like CoMemo~\cite{liu2025comemo} and ICoT~\cite{gao2025interleaved}. Notably, it effectively mitigates the degradation observed in perception tasks (e.g., MMT) common to standard fine-tuning.
When enhanced with GRPO, our method brings notable improvement with an overall score of 71.5\%, exceeding robust RL-tuned competitors like Euclid-8B~\cite{lian2025euclid} and PEARL-8B~\cite{zhang2025perceptual}.
Crucially, this efficacy proves highly consistent across model scales: on 4B model size, our approach also delivers a consistent +4.4\% overall improvement.

\begin{figure}[!bp]
    \centering
    \begin{minipage}[t]{0.48\textwidth}
        \centering
        \vspace{0pt} 
        \includegraphics[height=3.3cm]{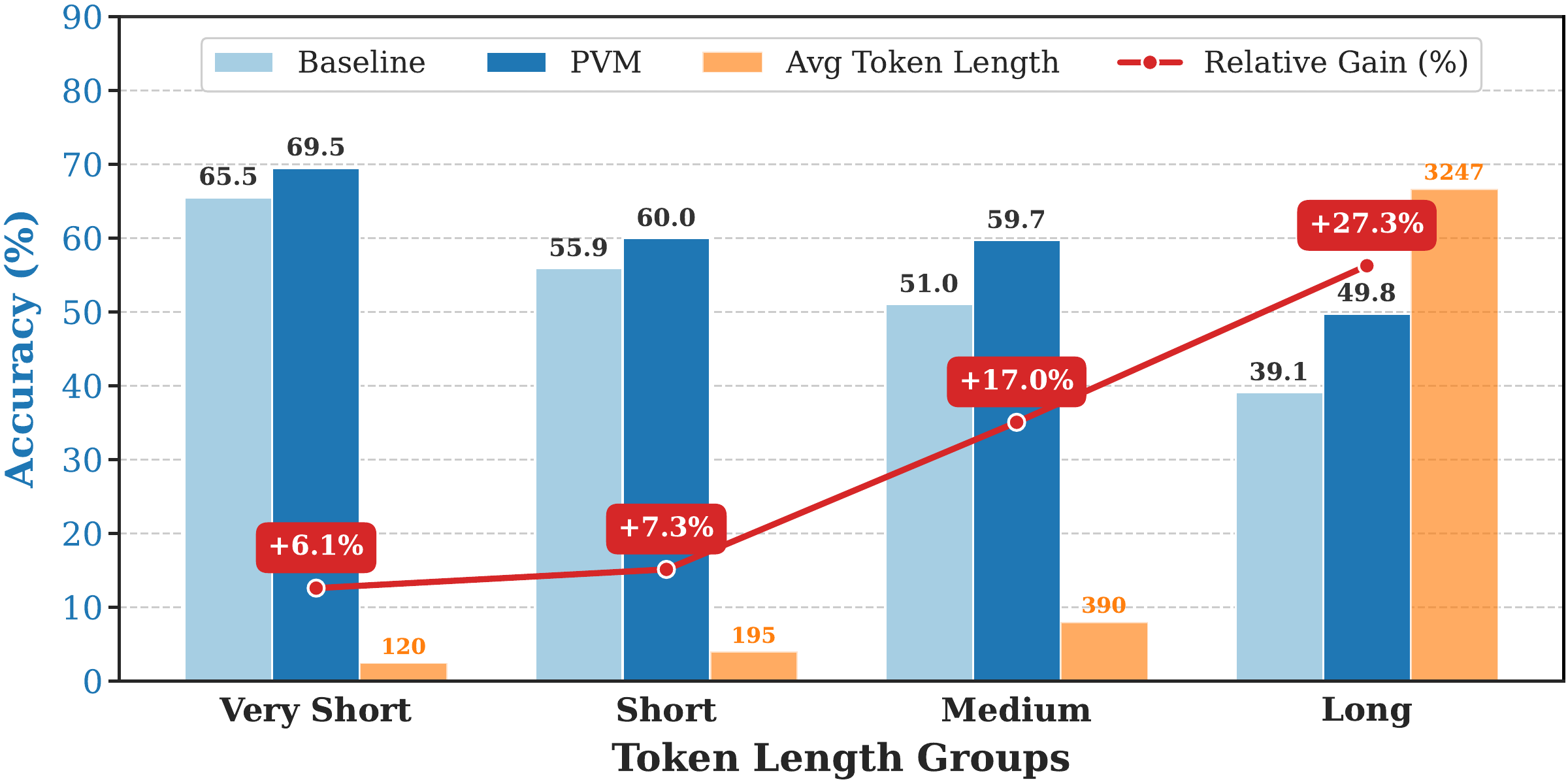} 
        \caption{\textbf{Performance Gain vs. Token Length.} The relative improvement scales with sequence length, surging to +27.3\% in the ``Long'' group. This suggests PVM helps structurally mitigate visual signal dilution in deep generation.}
        \label{fig:length_analysis}
    \end{minipage}
    \hfill
    \begin{minipage}[t]{0.48\textwidth}
        \centering
        \vspace{0pt} 
        \includegraphics[height=3.3cm]{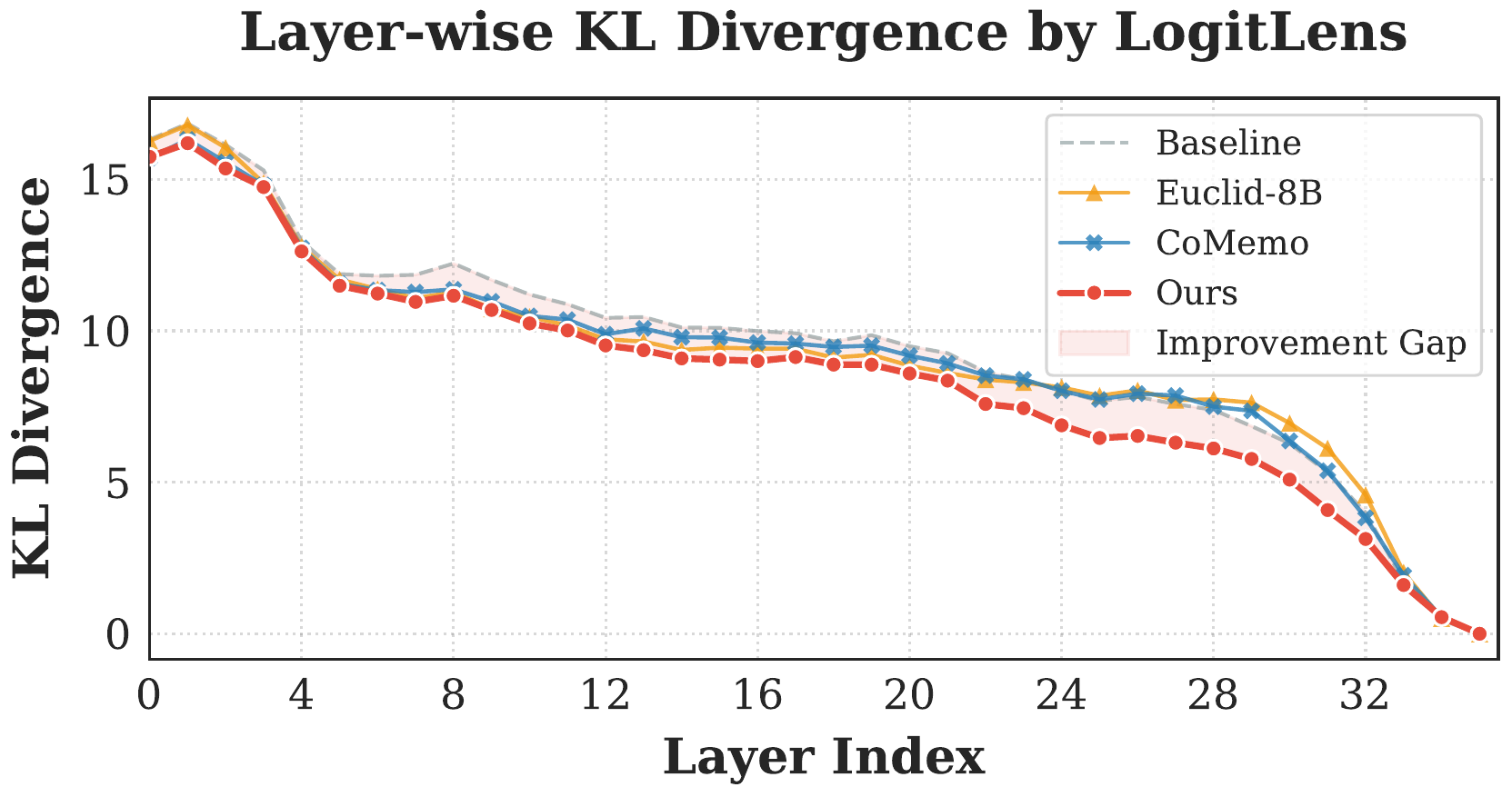} 
        \caption{\textbf{Layer-wise Prediction Convergence.} The steeper descent and distinct gap (shaded) suggest that PVM accelerates prediction convergence compared to strong baselines.}
        \label{fig:convergence}
    \end{minipage}
\end{figure}

\subsection{Robustness to Extended Generation}
\label{subsec:length_analysis}

To assess robustness in extended generation, we compare PVM-8B (SFT + GRPO) against the Qwen3-VL-8B-Instruct baseline on MathVerse$^\text{V}$, stratifying the samples into four equal-sized groups based on output length.
Figure~\ref{fig:length_analysis} shows a distinct positive correlation emerges between sequence length and relative gain.
While PVM yields a moderate +6.1\% improvement on ``Very Short'' samples, it unlocks a dramatic +27.3\% boost on ``Long'' ones, where the baseline suffers severe visual dilution.
To decouple this effect from problem difficulty, we provide further validation in Appendix~\ref{app:difficulty_control}.
This suggests PVM serves as a useful stabilizer: the deeper the reasoning chain, the more indispensable sustained visual retrieval becomes to prevent detachment from visual evidence.

\subsection{Mechanistic Analysis of PVM}
\label{subsec:convergence}

Following~\citet{cheng2026conditional}, to probe the internal mechanism of how PVM influences the model's predictive dynamics, we employ the \textbf{LogitLens} technique~\cite{nostalgebraist2020logitlens}. 
We quantify \textit{prediction readiness}~\cite{belrose2023eliciting, csordas2025language} by measuring the Kullback-Leibler (KL) divergence between intermediate layer representations and the final output distribution (see Appendix~\ref{sec:app_logitlens} for details).
Consistent with Section~\ref{subsec:empirical}, we conduct this analysis on the ``Blind Painter'' test to isolate the model's behavior under high visual dependency.

As shown in Figure~\ref{fig:convergence}, while baselines exhibit a gradual KL decline, PVM establishes a distinctly lower trajectory. 
Crucially, a significant ``Improvement Gap'' emerges after the initial injection (Layer 8) and widens across deeper layers. 
This suggests that PVM may help \textit{short-circuit} the information gathering process: by actively offloading visual retrieval to the parallel branch, the backbone accelerates its transition from perception to reasoning, thus speeding up convergence.

\subsection{Ablation Studies}
\label{subsec:ablation}

\begin{wraptable}{R}{0.5\textwidth} 
\vspace{-4.5mm}
\centering
\caption{\textbf{Ablation on Retrieval Source.} Replacing raw visual embeddings with processed hidden states causes severe performance degradation.}
\label{tab:ablation_source}
\scriptsize
\setlength\tabcolsep{2pt} 
\begin{tabular}{l|ccc|c} 
\toprule
\textbf{Retrieval Source ($K, V$)} & \textbf{MathVerse} & \textbf{MathVision} & \textbf{AI2D} & \textbf{Avg.}\\
\midrule
Baseline & 52.9 & 45.4 & 79.8 & 59.4 \\
Processed Hidden States & 27.9 & 14.1 & 58.2 & 33.4 \\
\rowcolor{LightCyan} \textbf{Visual Embeddings (Ours)} & \textbf{57.5} & \textbf{50.7} & \textbf{80.8} & \textbf{63.0} \\
\bottomrule
\end{tabular}
\end{wraptable}
\paragraph{Necessity of Raw Visual Retrieval.} To isolate the source of improvement, we replaced the raw visual embeddings with current processed hidden states ($Q, K, V \leftarrow \mathbf{x}$) and re-trained this variant under the identical two-stage pipeline. 
As shown in Table~\ref{tab:ablation_source},  this triggers a catastrophic collapse across reasoning benchmarks, indicating that re-integrating text-dominated hidden states creates a destructive self-reflexive loop that disrupts logical coherence. 
This validates that the true gains stem from PVM's retrieval design.

\paragraph{Injection Layer Selection.}
Based on the visual attention dynamics analyzed in Section~\ref{subsec:empirical}, we investigate the optimal placement for PVM modules. 
We compare our default \textit{Strided Strategy} (Layers 8, 16, 24) against two data-driven alternatives (see Appendix~\ref{sec:appendix_layer_selection} for detailed calculations): 
(1) \textit{Peak Attention}: Injecting into layers with the highest intrinsic visual attention mass (13, 17, 18) to reinforce existing signals; 
and (2) \textit{Max Decay}: Targeting layers that exhibit the sharpest drop in visual attention mass (14, 19, 22) to actively compensate for signal loss.

\begin{wraptable}{R}{0.41\textwidth} 
\centering
\vspace{-4.5mm} 
\caption{\textbf{Ablation on Injection Strategy.} Comparison of layer selection strategies for PVM modules on 8B model.}
\label{tab:ablation_layers}
\scriptsize
\setlength\tabcolsep{3pt} 
\begin{tabular}{l|c|cc|c} 
\toprule
\textbf{Selection Strategy} & \textbf{Layers} & \textbf{Gen.} & \textbf{Reas.} & \textbf{Avg.} \\
\midrule
Peak Attention & 13, 17, 18 & 72.9 & 60.9 & 68.4 \\
Max Decay & 14, 19, 22 & 74.2 & 61.2 & 69.3 \\
\rowcolor{LightCyan} \textbf{Strided (Ours)} & \textbf{8, 16, 24} & \textbf{75.2} & \textbf{63.0} & \textbf{70.6} \\
\bottomrule
\end{tabular}
\vspace{-4.5mm}
\end{wraptable}
As shown in Table~\ref{tab:ablation_layers}, \textit{Peak Attention} yields the lowest scores (68.4\%), indicating diminishing returns. While \textit{Max Decay} improves performance to 69.3\%, our \textit{Strided Strategy} proves superior (70.6\%). Unlike the clustered decay-based layers, our configuration spans the network's full depth. This global coverage ensures consistent visual grounding across diverse processing stages, yielding a +1.8\% reasoning gain over the decay-focused approach.

\paragraph{Extended Analyses.} The Appendix details latent dimension sensitivity (Appendix~\ref{app:ablation_dim}), a rigorous iso-parameter MLP control (Appendix~\ref{app:iso_parameter}), and inference overhead benchmarking (Appendix~\ref{app:overhead}).
\section{Conclusion}

In this work, we address {Visual Signal Dilution} through {Persistent Visual Memory (PVM)}. 
By establishing a dedicated parallel pathway for active retrieval, PVM helps decouple visual memory retention from the growing autoregressive context, improving access to visual evidence during extended generation. 
Empirically, our approach brings notable improvement across diverse benchmarks with negligible parameter overhead. 
Our findings underscore that shifting from passive retention to \textit{sustained, on-demand perception} is essential for robust, extended-horizon multimodal intelligence.

\bibliographystyle{abbrvnat}
\bibliography{main}







\newpage

\appendix

\section{Visual Attention Mass Heatmap Analysis}
\label{app:visual_attention}

\begin{figure}[!htbp]
    \centering
    \includegraphics[width=0.6\linewidth]{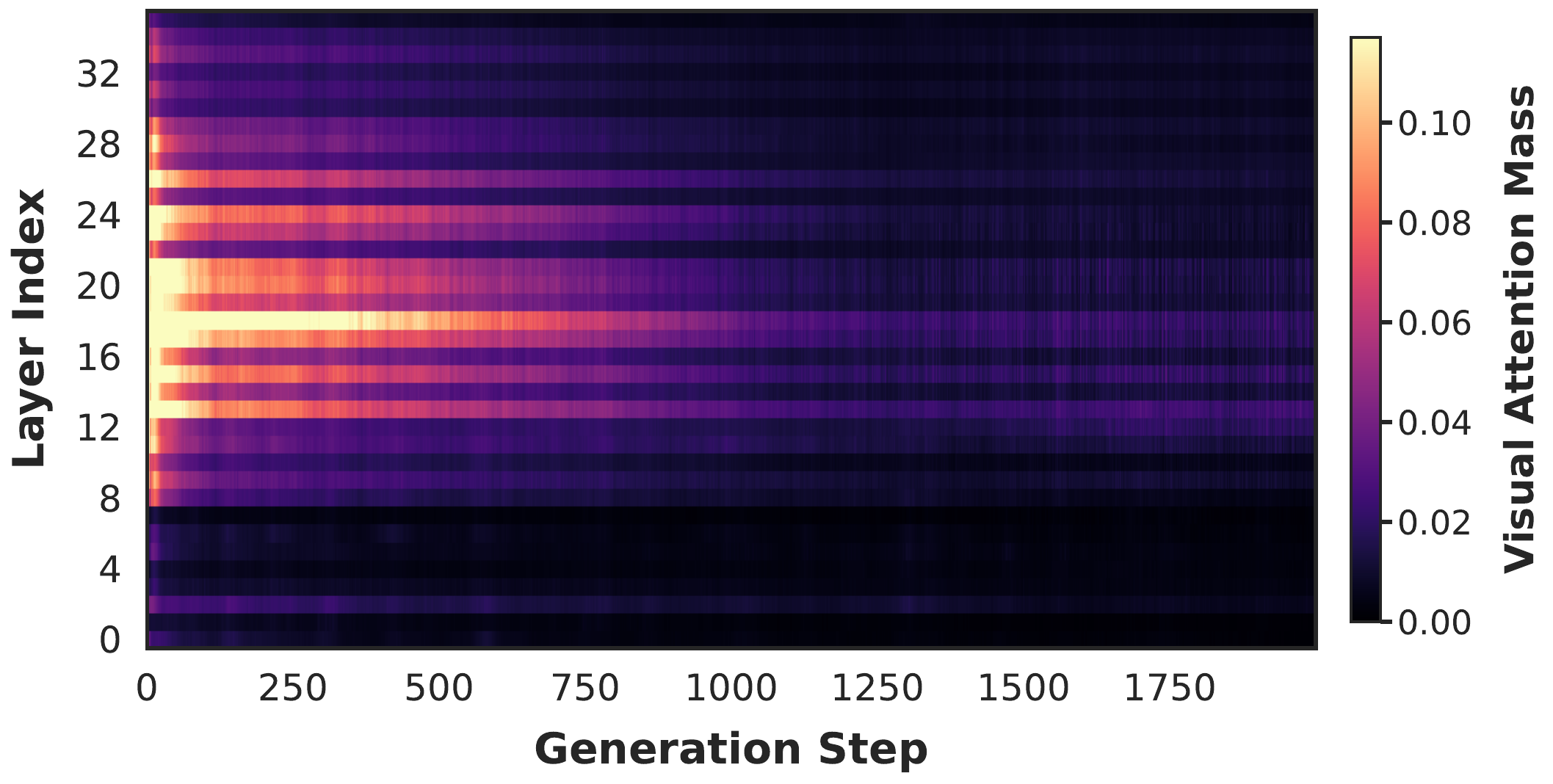}
    \caption{\textbf{Detailed Spatiotemporal Decay of Visual Attention.} 
    The heatmap illustrates the evolution of visual attention mass $\Omega_{\mathcal{V}}$ across all 36 layers of Qwen3-VL-8B-Instruct. The x-axis represents the number of generated text tokens, and the y-axis represents the layer index. 
    Darker regions indicate lower visual attention. 
    A distinct decay forms in the intermediate layers as the sequence grows, highlighting the structural necessity for the Persistent Visual Memory (PVM) module.}
    \label{fig:app_visual_heatmap}
\end{figure}

In Section~\ref{subsec:empirical}, we identified the phenomenon of \textit{Visual Signal Dilution}. To provide a granular understanding of this mechanism, we present the detailed spatiotemporal visualization of visual attention in Figure~\ref{fig:app_visual_heatmap}.

\paragraph{Visualization Methodology.}
The heatmap quantifies the \textbf{Visual Attention Mass} $\Omega_{\mathcal{V}}(l, t)$ for each Transformer layer $l$ at generation step $t$. 
$\Omega_{\mathcal{V}}$ is defined as the sum of probability mass allocated to the visual tokens $\mathcal{V}$ divided by the total probability mass (visual + textual).
The data is collected using the \textbf{``Blind Painter''} stress test (see Appendix~\ref{sec:appendix_prompt}) on \textbf{Qwen3-VL-8B-Instruct}, where the model is prompted to generate long-form descriptions, forcing a continuous demand for visual information.

\paragraph{Layer-wise Dynamics.}
The heatmap reveals that the attention decay is not uniform across the depth of the network. Instead, we observe three distinct architectural zones:

\begin{itemize}
    \item \textbf{Shallow Layers (0--7):} These layers exhibit consistently low visual attention ($\Omega_{\mathcal{V}} < 0.05$) throughout the generation. This aligns with findings in interpretability literature suggesting that early layers in LLMs primarily focus on local syntax and shallow textual features, requiring minimal multimodal integration.
    
    \item \textbf{Intermediate Layers (8--27):} This is the \textbf{critical reasoning zone} where the decay is most pronounced. Initially, these layers show high visual activation ($\Omega_{\mathcal{V}} > 0.10$), indicating their role in semantic grounding. However, as the sequence length $t$ increases, the visual mass in this region suffers a catastrophic collapse, dropping to negligible levels. This observation empirically justifies our decision to inject PVM modules specifically at layers 8, 16, and 24 to reinforce this specific region.
    
    \item \textbf{Deep Layers (28--35):} The final layers revert to a text-dominated state, likely focusing on output formatting and next-token prediction distributions.
\end{itemize}

\section{Discussion on the Fixed Local Query Assumption}
\label{sec:appendix_approximation}

In Section~\ref{sec:guarantee}, Theorem~\ref{thm:immunity} relies on the \textit{fixed local query assumption} to formally characterize the structural mitigation of visual dilution. Here, we provide a detailed discussion on the necessity and implications of this theoretical boundary.

In a rigorous autoregressive generation setting, the input hidden state to the PVM module at step $t$, denoted as $\mathbf{x}_t$, is implicitly a function of the entire preceding context (which includes both visual and growing textual tokens). Consequently, as the textual history expands, the query state $\mathbf{x}_t$ inevitably evolves over time. Due to this dynamic query drift, the final PVM output $\mathbf{h}_{\mathrm{pvm}}$ cannot be absolutely invariant to $t$ across the full global decoding trajectory.

However, the core phenomenon of visual dilution (as identified in Theorem~\ref{thm:dilution}) is primarily driven by the explicit, unconstrained expansion of the textual partition term $Z_{\mathcal{T}}$ within the Softmax denominator. To mathematically isolate this structural bottleneck from the natural semantic evolution of the queries, we introduce the condition of a fixed local hidden state $\mathbf{x}$.

By conditioning on a fixed local query $\mathbf{x}$ and treating the visual set $\mathcal{V}$ as constant keys, we can pinpoint the specific impact of the sequence length $t$ on the attention normalization. Under this approximation, the partial derivative vanishes ($\frac{\partial \|\mathbf{h}_{\mathrm{pvm}}\|}{\partial t} = 0$). While this result does not claim exact global invariance of the visual representations during real-world decoding, it rigorously proves that PVM's partition function is structurally isolated from the expanding textual mass. This decoupling mechanism algebraically shields the visual branch from the probability competition induced by long texts, thereby establishing a persistent visual memory.

\section{Implementation Details}
\label{app:training_details}

\begin{table}[!htbp]
\centering
\caption{\textbf{Detailed Hyperparameter Settings.} We report the specific configurations for both the SFT alignment phase and the GRPO reinforcement learning phase.}
\label{tab:hyperparameters}
\scriptsize
\renewcommand{\arraystretch}{1.2}
\setlength\tabcolsep{6pt}
\begin{tabular}{l|c|c}
\toprule
\textbf{Hyperparameter} & \textbf{Stage I: Alignment (SFT)} & \textbf{Stage II: Refinement (GRPO)} \\
\midrule
\multicolumn{3}{l}{\textit{Optimization}} \\
\midrule
Optimizer & AdamW & AdamW \\
Learning Rate & 1e-4 & 1e-6 \\
LR Scheduler & Cosine & Constant \\
Warmup Ratio & 0.1 & 0.0 \\
Weight Decay & 0.0 & 0.0 \\
Gradient Accumulation & 8 & 8 \\
Max Gradient Norm & 1.0 & 1.0 \\
\midrule
\multicolumn{3}{l}{\textit{Batch \& Sequence}} \\
\midrule
Global Batch Size & 64 & 64 \\
Per-Device Batch Size & 1 & 1 \\
Max Completion Length & None & 16384 \\
Group Size ($G$) for GRPO & N/A & 8 \\
KL Coefficient & N/A & 0.0 \\
\midrule
\multicolumn{3}{l}{\textit{Module Status (Freeze/Train)}} \\
\midrule
Vision Encoder & Frozen & Frozen \\
LLM Backbone & Frozen & Trainable \\
PVM Modules & Trainable & Trainable \\
Projector & Frozen & Frozen \\
\bottomrule
\end{tabular}
\end{table}

In this section, we provide the comprehensive hyperparameter settings and system configurations used to train the PVM-enhanced Qwen3-VL models.

\paragraph{System Infrastructure.}
All experiments were conducted on a high-performance computing cluster equipped with \textbf{8 $\times$ NVIDIA H200 GPUs} (141GB VRAM per GPU). 
Our implementation builds upon the PyTorch framework, utilizing Hugging Face's \texttt{transformers}~\cite{wolf-etal-2020-transformers} and \texttt{trl}~\cite{vonwerra2020trl} libraries. 
To maximize training efficiency and support long-context processing, we employ \textbf{DeepSpeed} optimization strategies tailored to each stage: \textbf{ZeRO-2} is used for the SFT alignment phase, while \textbf{ZeRO-3} is adopted for the GRPO refinement phase to manage the increased memory overhead of group sampling. Additionally, we enable \textbf{FlashAttention-2} for accelerated attention computation, and gradient checkpointing is applied to the language backbone to further reduce the memory footprint.

\paragraph{Model Configuration.}
To target the critical reasoning depths identified in our analysis, we inject PVM modules into the intermediate Transformer layers. Specifically, we select layer indices $\{8, 16, 24\}$ for the \textbf{8B model} and $\{5, 11, 17\}$ for the \textbf{4B model}. 
The bottleneck latent dimension $d'$ is set to 512 by default (as verified in Section~\ref{subsec:ablation}). The gating scalar $\alpha$ is initialized to 0 to ensure a stable warm-up, allowing the model to gradually incorporate the visual memory branch without disrupting the pre-trained autoregressive priors.

\paragraph{Data Curation.}
For the SFT phase, the 526k samples in $\mathcal{D}_{\text{sft}}$ are specifically filtered based on visual centricity and answer clarity from the broader OpenMMReasoner-SFT-874K dataset~\cite{zhang2025openmmreasoner}. For the GRPO refinement phase, the 3.6k queries in $\mathcal{D}_{\text{rl}}$ are curated by generating 8 reasoning rollouts per query. We retain only the samples that exhibited the strongest learning signals to ensure robust policy optimization.

\paragraph{Hyperparameter Settings.}
We adopt a two-stage training strategy: \textbf{Stage I (Visual Memory Alignment)} focuses on initializing the PVM parameters, while \textbf{Stage II (Policy Refinement)} employs GRPO to optimize the model for complex reasoning. The detailed hyperparameters for both stages are listed in Table~\ref{tab:hyperparameters}.

\section{Decoupling Sequence Length from Problem Difficulty}
\label{app:difficulty_control}

In Section~\ref{subsec:length_analysis}, we demonstrated a distinct positive correlation between the sequence length of the generation and the relative performance gain achieved by PVM. However, in autoregressive generation, a natural confounding factor arises: queries that require longer reasoning outputs often correspond to intrinsically harder tasks. To rigorously isolate the effect of sequence length from problem difficulty, we conduct an additional controlled analysis on the MathVerse$^\text{V}$ benchmark.

\textbf{Methodology.} We employed Gemini-3-Flash~\cite{comanici2025gemini} to automatically evaluate and stratify the test samples into discrete difficulty tiers (``Easy'', ``Medium'', and ``Hard''). We excluded the ``Easy'' subset from this analysis due to its statistically insignificant sample size ($n < 10$). Within each of the remaining ``Medium'' and ``Hard'' tiers, we further subdivided the samples into ``Short'', ``Medium'', and ``Long'' equal-sized groups based on their generated token lengths to observe the length-induced dynamics under a fixed difficulty condition.

\begin{table}[h]
\centering
\caption{\textbf{Performance Gain Controlling for Problem Difficulty.} We report the absolute accuracy gain of PVM over the baseline. The numbers in parentheses indicate the average output token length of the samples within that specific group.}
\scriptsize
\label{tab:difficulty_control}
\renewcommand{\arraystretch}{1.2}
\begin{tabular}{lccc}
\toprule
\textbf{Difficulty Tier} & \textbf{Gain (Short)} & \textbf{Gain (Medium)} & \textbf{Gain (Long)} \\
\midrule
\textbf{Medium} & +4.00\% \textit{(120)} & +11.71\% \textit{(192)} & +15.71\% \textit{(1435)} \\
\textbf{Hard}   & +4.27\% \textit{(198)} & +5.98\% \textit{(480)}  & +9.40\% \textit{(3971)} \\
\bottomrule
\end{tabular}
\end{table}

\textbf{Results and Analysis.} The results of this stratified analysis are presented in Table~\ref{tab:difficulty_control}. Crucially, even when strictly controlling for task difficulty, PVM's absolute performance gain consistently scales with sequence length across both the Medium and Hard tiers. For instance, within the ``Medium'' difficulty tier, the improvement surges from +4.00\% on short sequences to +15.71\% on long sequences. 

This consistent upward trend helps decouple task difficulty from the observed gains and provides additional support for our hypothesis that PVM is more beneficial in longer generations, likely due to its structural ability to mitigate length-associated visual dilution.

\section{LogitLens Analysis Formulation}
\label{sec:app_logitlens}

In Section~\ref{subsec:convergence}, we utilize the LogitLens technique to visualize the layer-wise convergence of the model's predictions. This section details the mathematical formulation used for this analysis.

\paragraph{Projection to Vocabulary Space.}
Let $\mathcal{M}$ be an $L$-layer Transformer model with a vocabulary size $V$. Let $\mathbf{h}_{\ell} \in \mathbb{R}^{d}$ denote the hidden state output by the $\ell$-th layer. 
Standard LogitLens projects this intermediate state directly into the vocabulary space using the pre-trained, frozen language modeling head (unembedding matrix) $\mathbf{E} \in \mathbb{R}^{V \times d}$:
\begin{equation}
    P_{\ell} = \text{softmax}(\mathbf{E} \mathbf{h}_{\ell})
\end{equation}
where $P_{\ell} \in \mathbb{R}^{V}$ represents the probability distribution over the vocabulary as predicted by the $\ell$-th layer.

\paragraph{Quantifying Convergence.}
To measure how close an intermediate representation is to the model's final decision, we compute the Kullback-Leibler (KL) divergence between the intermediate distribution $P_{\ell}$ and the final output distribution $P_{\text{final}} = P_{L}$:
\begin{equation}
    D_{\text{KL}}(P_{\text{final}} \parallel P_{\ell}) = \sum_{v=1}^{V} P_{\text{final}}(v) \log \left( \frac{P_{\text{final}}(v)}{P_{\ell}(v)} \right)
\end{equation}
A lower KL divergence value indicates that the features at layer $\ell$ have already encoded sufficient semantic information to approximate the final output. 
By tracking $D_{\text{KL}}$ across $\ell \in \{1, \dots, L\}$, we construct the convergence trajectory visualized in Figure~\ref{fig:convergence}.

\section{Detailed Analysis of Injection Layer Selection}
\label{sec:appendix_layer_selection}

To determine the optimal insertion points for the Persistent Visual Memory (PVM), we conducted a quantitative profiling of the visual attention distribution across all transformer layers of the Qwen3-VL-8B-Instruct backbone.

\begin{figure}[h]
\centering
\includegraphics[width=0.6\linewidth]{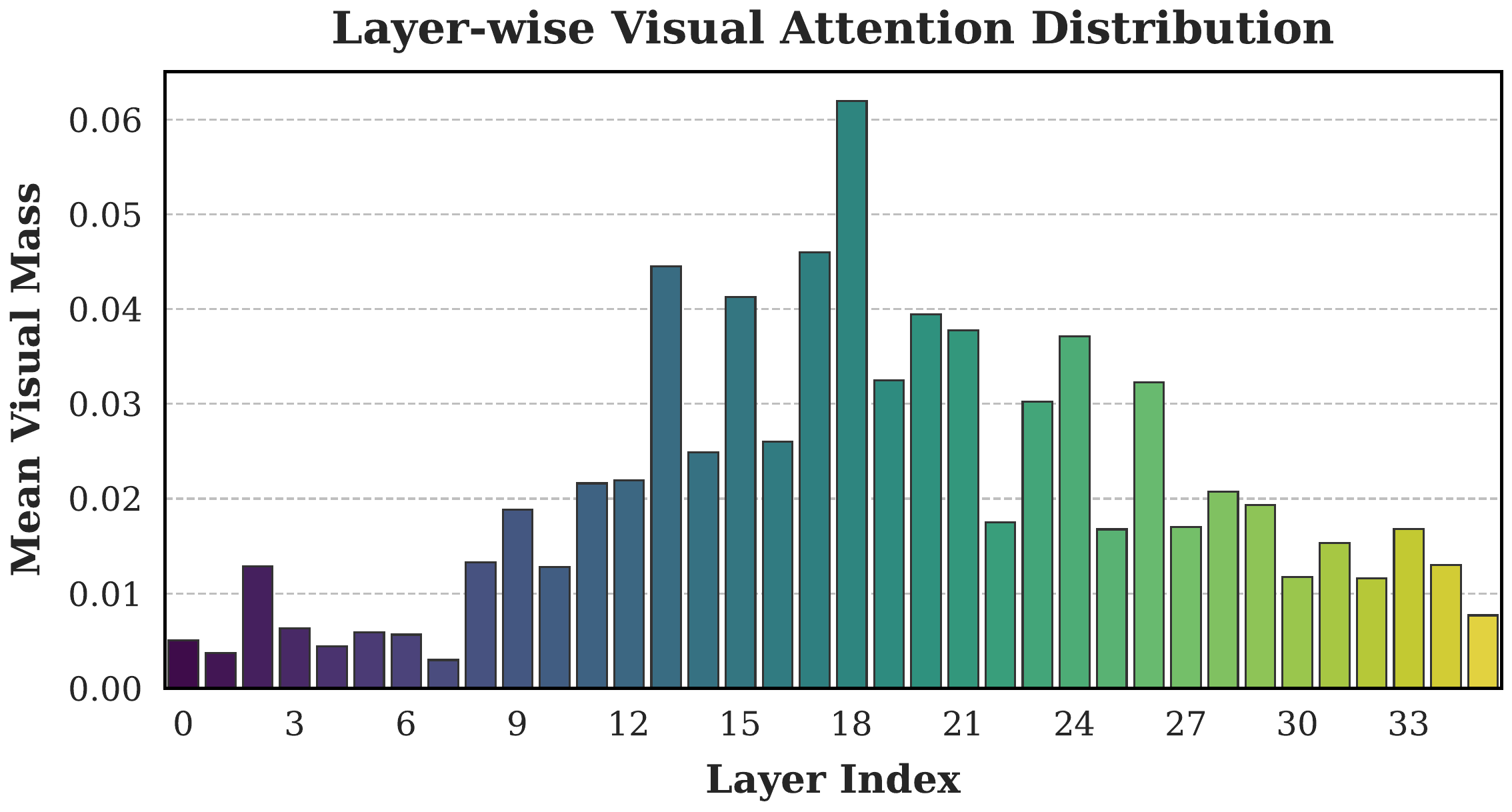} 
\caption{\textbf{Layer-wise Distribution of Mean Visual Attention Mass.} The bar chart visualizes the aggregate attention weight assigned to visual tokens at each layer. We observe a characteristic "Rise-Peak-Decay" pattern, guiding our data-driven injection strategies.}
\label{fig:layer_profile}
\end{figure}

\paragraph{Methodology.}
We define the \textit{Mean Visual Mass} for each layer by averaging the visual attention mass (defined in Eq.~\ref{eq:standard_attn}) over all samples and generation steps. Figure~\ref{fig:layer_profile} illustrates this distribution. Based on this profile, we formulated three selection strategies:

\begin{enumerate}
\item \textbf{Peak Attention (Reinforcement Strategy):} This strategy hypothesizes that PVM should bolster layers where the model is already actively seeking visual information. We selected the top-3 layers with the highest absolute magnitude:
\begin{equation}
\mathcal{L}_{\text{peak}} = \text{Top-3}_{\ell}(\bar{\Omega}_{\mathcal{V}}^\ell) \rightarrow {13, 17, 18}
\end{equation}
As seen in Figure~\ref{fig:layer_profile}, Layer 18 represents the global maximum, with Layers 13 and 17 forming secondary peaks.

\item \textbf{Max Decay (Compensation Strategy):} This strategy aims to ``rescue'' the visual signal at the precise moments it suffers the most severe attenuation. We calculated the discrete derivative of the attention mass, $\Delta^\ell = \bar{\Omega}_{\mathcal{V}}^{\ell-1} - \bar{\Omega}_{\mathcal{V}}^\ell$. We selected layers corresponding to the largest drops (positive $\Delta^\ell$):
\begin{equation}
    \mathcal{L}_{\text{decay}} = \text{Top-3}_{\ell}(\Delta^\ell) \rightarrow \{14, 19, 22\}
\end{equation}
Referring to Figure~\ref{fig:layer_profile}, significant drops are observable immediately after the peaks: from Layer 13 to 14, 18 to 19, and 21 to 22. These layers represent ``bottlenecks'' where visual context is rapidly discarded in favor of textual processing.

\item \textbf{Strided Strategy (Global Coverage):} 
Observation of Figure~\ref{fig:layer_profile} reveals that layers 0--7 possess negligible visual mass, serving primarily as a textual warm-up phase. Significant visual processing initiates at Layer 8. 
Therefore, we adopted a uniform strided placement starting from this active onset: $\mathcal{L}_{\text{stride}} = \{8, 16, 24\}$. This strategy avoids clustering modules in the middle and ensures visual injection is distributed evenly across the shallow, middle, and deep reasoning blocks.
\end{enumerate}

\section{Impact of Latent Dimension Size}
\label{app:ablation_dim}

In this section, we analyze the sensitivity of the model performance to the PVM bottleneck size ($d'$). This hyperparameter controls the capacity of the Persistent Visual Memory and the compression rate of the visual features.

\begin{table}[!htbp]
\centering
\caption{\textbf{Ablation on Latent Dimension.} We analyze the impact of the PVM bottleneck size ($d'$). A dimension of 512 achieves optimal performance and parameter efficiency.}
\label{tab:ablation_dim}
\scriptsize
\setlength\tabcolsep{5pt} 
\begin{tabularx}{0.6\linewidth}{Y|YY|Y}
\toprule
\textbf{Latent Dim ($d'$)} & \textbf{General} & \textbf{Reasoning} & \textbf{Avg.} \\
\midrule
\rowcolor{LightCyan} \textbf{512 (Ours)} & \textbf{75.2} & \textbf{63.0} & \textbf{70.6} \\
1024  & 73.8 & 61.4 & 69.2 \\
2048  & 74.7 & 61.6 & 69.8 \\
\bottomrule
\end{tabularx}
\end{table}

As shown in Table~\ref{tab:ablation_dim}, increasing the latent dimension to 1024 or 2048 does not lead to performance improvements; in fact, it results in a slight regression compared to our default setting of $d'=512$. 

\paragraph{Analysis of Data-Capacity Mismatch.}
We attribute this phenomenon to the constraints of the available fine-tuning data scale. Expanding the latent dimension significantly increases the parameter count of PVM. 
According to scaling laws~\cite{kaplan2020scalinglawsneurallanguage, hoffmann2022trainingcomputeoptimallargelanguage}, larger parameter spaces require proportionally larger supervision signals to be effectively optimized. 
Given the size of our current SFT dataset, the model likely struggles to fully saturate the capacity of higher-dimensional bottlenecks, potentially leading to optimization difficulties or overfitting to noise. 
Consequently, $d'=512$ provides the optimal balance, offering sufficient representational capacity for visual retrieval while remaining compact enough to be robustly trained with the available data.

\section{Iso-Parameter Control Analysis}
\label{app:iso_parameter}

To rigorously verify that the performance improvements of our Persistent Visual Memory (PVM) stem from the active visual retrieval mechanism rather than a mere increase in parameter capacity, we conduct an iso-parameter control experiment.

\paragraph{Baseline Design.}
We design a parallel MLP baseline that exactly matches the parameter count of the integrated PVM modules. Crucially, this variant removes the visual cross-attention mechanism, meaning it cannot retrieve raw visual signals and relies solely on processed hidden states. To ensure a strictly fair comparison, this iso-parameter baseline is trained from scratch using the identical two-stage pipeline (SFT followed by GRPO) as our default PVM model.

\begin{table}[!htbp]
\centering
\caption{\textbf{Iso-Parameter Control Results.} Comparison between our PVM model and an iso-parameter MLP baseline across 8 complex reasoning benchmarks. Both models share the exact same parameter count and are trained under the identical SFT+GRPO pipeline.}
\label{tab:iso_parameter_control}
\scriptsize
\setlength\tabcolsep{3.5pt} 
\begin{tabular}{l|cccccccc|c}
\toprule
\textbf{Model Setup} & \textbf{MMMU} & \textbf{MMBench\_CN} & \textbf{MMBench\_EN} & \textbf{MMStar} & \textbf{MMT} & \textbf{MathVerse} & \textbf{MathVision} & \textbf{AI2D} & \textbf{Avg.} \\
\midrule
SFT + GRPO & 60.7 & 88.8 & 87.9 & 68.6 & 54.2 & 58.5 & 48.0 & 79.6 & 68.3 \\
MLP (SFT+GRPO) & 63.3 & 88.8 & 88.7 & 70.0 & 55.0 & 58.0 & 48.7 & 79.4 & 69.0 \\
\rowcolor{LightCyan} \textbf{PVM-8B (SFT+GRPO)} & \textbf{67.3} & \textbf{91.2} & \textbf{89.4} & \textbf{71.6} & \textbf{58.3} & \textbf{59.8} & \textbf{51.3} & \textbf{82.8} & \textbf{71.5} \\
\bottomrule
\end{tabular}
\end{table}

\paragraph{Results and Analysis.}
As shown in Table~\ref{tab:iso_parameter_control}, despite possessing the exact same parameter capacity and undergoing identical RL optimization, the MLP baseline consistently underperforms \textbf{PVM-8B (SFT+GRPO)} across all 8 evaluation benchmarks, where PVM achieves a $2.5\%$ higher average score.

This comprehensive comparison confirms that the added parameters in our architecture do not merely act as a regularizer or a passive capacity booster for the language backbone. Instead, the substantial performance gains are fundamentally attributed to the PVM's ability to dynamically retrieve and integrate preserved visual signals during the reasoning process.

\section{Computational Overhead Analysis}
\label{app:overhead}

To strictly quantify the inference cost introduced by the Persistent Visual Memory (PVM) module, we conducted a benchmarking study comparing the PVM-enhanced model against the standard Qwen3-VL-8B-Instruct baseline.

\paragraph{Benchmarking Protocol.}
We implemented a low-level profiling script using the \texttt{PyTorch} framework and \texttt{transformers} library. To simulate a realistic interactive environment, we measured performance using a \textbf{streaming generation} setup with a batch size of 1.
The testing configuration is as follows:
\begin{itemize}
    \item \textbf{Environment:} A single NVIDIA H200 GPU (141GB VRAM).
    \item \textbf{Precision:} \texttt{bfloat16} mixed precision with \texttt{FlashAttention-2} enabled.
    \item \textbf{Generation Strategy:} Greedy decoding (\texttt{do\_sample=False}) to ensure deterministic latency measurements.
    \item \textbf{Warm-up:} A dry-run generation (5 tokens) was executed prior to measurement to eliminate kernel initialization overhead and JIT compilation artifacts.
\end{itemize}

\paragraph{Metrics Definition.}
We focus on two critical latency metrics derived from the token streaming timestamps $t_0, t_1, \dots, t_N$, where $t_0$ is the start time and $t_i$ is the arrival time of the $i$-th token:
\begin{itemize}
    \item \textbf{Time Per Output Token (TPOT):} Measures the decoding latency per step, strictly excluding the prefill phase. This represents the autoregressive generation speed: 
    \begin{equation}
        \text{TPOT} = \frac{t_N - t_1}{N - 1}
    \end{equation}
    \item \textbf{Throughput:} Defined as the inverse of TPOT ($1/\text{TPOT}$), representing the generation speed in tokens per second (tokens/s).
\end{itemize}

\paragraph{Quantitative Results.}
Table~\ref{tab:overhead_detailed} summarizes the performance comparison. 
The inclusion of PVM introduces a fixed computational graph expansion due to the parallel branch (Projection $\rightarrow$ Cross-Attention $\rightarrow$ Fusion). 
However, due to our parameter-efficient bottleneck design, this overhead is minimal.
The TPOT increases by only \textbf{1.18 ms}, resulting in a throughput reduction of 4.6\%. 
This confirms that PVM provides a highly favorable trade-off, delivering significant gains (as shown in Section~\ref{subsec:main_results}) with negligible impact on real-time inference capability.

\begin{table}[h]
\centering
\caption{\textbf{Inference Speed Comparison.} Evaluated on a single H200 GPU with bfloat16 precision. PVM maintains high-speed generation with marginal latency overhead.}
\label{tab:overhead_detailed}
\scriptsize
\renewcommand{\arraystretch}{1.2}
\setlength\tabcolsep{10pt}
\begin{tabular}{l|c|c|c}
\toprule
\textbf{Metric} & \textbf{Baseline (Qwen3-VL)} & \textbf{Ours (PVM-Enhanced)} & \textbf{Delta} \\
\midrule
Decoding Throughput & 41.18 tokens/s & 39.28 tokens/s & -4.61\% \\
Time Per Output Token (TPOT) & 24.28 ms & 25.46 ms & +1.18 ms \\
\bottomrule
\end{tabular}
\end{table}

\section{PVM Inference Algorithm}
\label{app:algorithm}

In this section, we provide the pseudocode for the forward pass of a Transformer decoder block integrated with the Persistent Visual Memory (PVM) module. Algorithm~\ref{alg:pvm_forward} details the exact computational flow, highlighting the parallel bifurcation between the static reasoning path (FFN) and the dynamic visual retrieval path (PVM).

The algorithm explicitly demonstrates two critical mechanisms designed to preserve signal fidelity:
\begin{itemize}
    \item \textbf{Latent Compression:} The projection of queries and keys into a concentrated latent space ($d'$) to distill core visual semantics and filter redundancy.
    \item \textbf{Visual Silencing:} The application of the mask $\mathcal{M}_{\mathrm{txt}}$ to ensure that only textual tokens trigger visual retrieval, preventing redundant self-referencing by visual tokens.
\end{itemize}

\begin{algorithm}[!h]
  \caption{Forward Pass of PVM-Enhanced Transformer Block}
  \label{alg:pvm_forward}
  \begin{algorithmic}[1]
    \Require Hidden states $\mathbf{x} \in \mathbb{R}^{L \times d}$ (Sequence length $L$, Model dim $d$)
    \Require Visual Context $\mathbf{V}_{\mathrm{img}} \in \mathbb{R}^{M \times d}$ (Visual tokens $M$)
    \Require Visual Silencing Mask $\mathcal{M}_{\mathrm{txt}} \in \{0, 1\}^L$ (1 for text, 0 for image)
    \Require Learnable Gate $\alpha$ initialized to 0
    
    \Statex \textcolor{gray}{// --- Stage 1: Standard Self-Attention ---}
    \State $\mathbf{x}_{\mathrm{norm}} \leftarrow \mathrm{RMSNorm}(\mathbf{x})$
    \State $\mathbf{h}_{\mathrm{attn}} \leftarrow \mathrm{MHSA}(\text{Query}=\mathbf{x}_{\mathrm{norm}}, \text{KV}=\text{Cache} \cup \mathbf{x}_{\mathrm{norm}})$
    \State $\mathbf{x} \leftarrow \mathbf{x} + \mathbf{h}_{\mathrm{attn}}$ \Comment{Residual Connection}
    
    \Statex
    \Statex \textcolor{gray}{// --- Stage 2: Parallel Bifurcation ---}
    \State $\mathbf{x}_{\mathrm{norm}} \leftarrow \mathrm{RMSNorm}(\mathbf{x})$
    
    \State \textbf{Path A: Static Reasoning (Frozen FFN)}
    \State \quad $\mathbf{h}_{\mathrm{ffn}} \leftarrow \mathrm{FFN}(\mathbf{x}_{\mathrm{norm}})$
    
    \State \textbf{Path B: Active Visual Retrieval (PVM)}
    \Statex \quad \textcolor{gray}{// B1. Compression to Latent Space ($d'$)}
    \State \quad $\mathbf{q}_{\mathrm{lat}} \leftarrow \mathbf{x}_{\mathrm{norm}} \mathbf{W}_{\mathrm{down}}^{\mathrm{txt}}$
    \State \quad $\mathbf{K}_{\mathrm{lat}}, \mathbf{V}_{\mathrm{lat}} \leftarrow \mathbf{V}_{\mathrm{img}} \mathbf{W}_{\mathrm{down}}^{\mathrm{vis}}$
    
    \Statex \quad \textcolor{gray}{// B2. Gated Cross-Attention \& Latent FFN}
    \State \quad $\mathbf{h}_{\mathrm{cross}} \leftarrow \mathrm{CrossAttn}(\text{Q}=\mathbf{q}_{\mathrm{lat}}, \text{K}=\mathbf{K}_{\mathrm{lat}}, \text{V}=\mathbf{V}_{\mathrm{lat}})$
    \State \quad $\mathbf{h}_{\mathrm{lat}} \leftarrow \mathbf{h}_{\mathrm{cross}} + \mathrm{FFN}_{\mathrm{lat}}(\mathrm{RMSNorm}(\mathbf{h}_{\mathrm{cross}}))$
    
    \Statex \quad \textcolor{gray}{// B3. Restoration \& Gating}
    \State \quad $\mathbf{h}_{\mathrm{pvm}} \leftarrow \mathbf{h}_{\mathrm{lat}} \mathbf{W}_{\mathrm{up}}$
    \State \quad $\mathbf{injection} \leftarrow (\alpha \cdot \mathbf{h}_{\mathrm{pvm}}) \odot \mathcal{M}_{\mathrm{txt}}$ \Comment{Apply Visual Silencing}
    
    \Statex
    \Statex \textcolor{gray}{// --- Stage 3: Unified Fusion ---}
    \State $\mathbf{y} \leftarrow \mathbf{x} + \mathbf{h}_{\mathrm{ffn}} + \mathbf{injection}$
    
    \Statex
    \Ensure Output hidden states $\mathbf{y} \in \mathbb{R}^{L \times d}$
  \end{algorithmic}
\end{algorithm}

\section{Prompt Templates}
\label{sec:appendix_prompt}

In this section, we provide the exact prompt templates utilized in our empirical analysis and training phases.

\paragraph{Visual Stress Test.}
To empirically verify the \textit{Visual Signal Dilution} phenomenon (Section~\ref{subsec:empirical}), we designed the \textbf{``Blind Painter''} prompt. As shown in the \textit{``Blind Painter'' Test Template}, this directive is engineered to be purposefully demanding. By explicitly requesting ``every single brushstroke'' and a ``monolithic wall of text,'' we force the model to generate extended sequences that maintain a continuous, high-intensity dependency on the visual input. This prevents the model from relying on generic hallucinations and isolates its ability to sustain visual attention over deep generation.

\paragraph{Structured Reasoning.}
Following the setting in OpenMMReasoner-SFT-874K~\cite{zhang2025openmmreasoner}, we employ a structured system prompt shown in the \textit{Reasoning Template} for the Policy Refinement stage and the evaluation. This format enforces a Chain-of-Thought (CoT) process, requiring the model to explicitly generate an internal monologue within \texttt{<think>} tags before producing the final answer.

\definecolor{rliableblue}{HTML}{77AADD}
\newcommand{\placeholder}[1]{\textcolor{red}{\{#1\}}}

\begin{tcolorbox}[colback=rliableblue!10!white,colframe=black,boxrule=1pt,boxsep=2pt,top=3pt,bottom=3pt,left=2pt,right=2pt]
\begin{center}
\textbf{``Blind Painter'' Test Template}
\end{center}
\textbf{USER:} \\ 
Describe this image for a blind painter who needs to recreate it perfectly without ever seeing it.
You must describe every single brushstroke, every strand of hair, every dust mote, and every subtle shift in light.
Start from the macroscopic composition and drill down into the microscopic details.
Describe the exact curvature of lines, the precise hex codes of colors, and the interplay of light and shadow.
Do not leave any pixel unaccounted for. Write a monolithic wall of text describing this image.
Do not stop until you reach the token limit.
\end{tcolorbox}

\begin{tcolorbox}[colback=rliableblue!10!white,colframe=black,boxrule=1pt,boxsep=2pt,top=3pt,bottom=3pt,left=2pt,right=2pt]
\begin{center}
\textbf{Reasoning Template}
\end{center}
\textbf{SYSTEM:} \\ You are a helpful assistant. When the user asks a question, your response must include two parts: first, the reasoning process enclosed in {\textless think\textgreater...\textless/think\textgreater} tags, then the final answer enclosed in {\textless answer\textgreater...\textless/answer\textgreater} tags. Please provide a clear, concise response within {\textless answer\textgreater \textless/answer\textgreater} tags that directly addresses the question. \\
\\
\textbf{USER:} \\ \placeholder{question}
\end{tcolorbox}

\section{Societal Impacts}
\label{app:societal_impacts}

Our work introduces Persistent Visual Memory (PVM) to improve the visual fidelity of Large Vision-Language Models (LVLMs) during extended generation. By structurally mitigating visual hallucinations, PVM positively contributes to the reliability of LVLMs in applications like scientific reasoning and visual assistants. As a foundational architectural improvement, our method does not inherently introduce new or specific societal risks beyond those already associated with base LVLMs. However, like any advanced generative AI, models equipped with PVM could still be misused to generate convincing deceptive content or inherit biases from their pre-training data. Addressing these general risks relies on standard safety alignment and responsible deployment practices for the backbone models.

\section{Limitations and Future Work}
\label{app:limitations}

While Persistent Visual Memory (PVM) effectively mitigates visual signal dilution, we note a few directions for future exploration. First, our empirical evaluation currently focuses on the representative Qwen3-VL (4B and 8B) series. Although PVM's parallel design is theoretically backbone-agnostic, validating its efficacy across a broader range of LVLM architectures and larger parameter scales is a natural next step. Second, as discussed in Appendix~\ref{sec:appendix_approximation}, our theoretical guarantees utilize a fixed local query assumption to rigorously isolate the dilution effect. Modeling the exact global dynamics of query drift over extremely long horizons could provide further mechanistic insights. Finally, we focus on mitigating dilution for static visual contexts in this work; extending the persistent memory mechanism to dynamic streaming inputs, such as long-form video understanding, presents an exciting avenue for future research.




\end{document}